\documentclass[10pt,twocolumn,letterpaper]{article}

\usepackage[pagenumbers]{cvpr} % To force page numbers, e.g. for an arXiv version

\usepackage{amsfonts}
\usepackage{amsmath}
\usepackage{amssymb}

\usepackage{xspace}
\usepackage{xcolor}
\usepackage{comment}
\usepackage{nicematrix}
\usepackage{colortbl}
\usepackage{pgf}
\usepackage{tikz}
\usepackage{graphicx}

\newtheorem{remark}{Remark}
\newtheorem{definition}{Definition}
\newtheorem{lemma}{Lemma}

\newenvironment{proof}{\paragraph{Proof:}}{\hfill$\square$}

\usepackage{algpseudocode,algorithm,algorithmicx}

\usepackage{booktabs}
\usepackage{subcaption} 
\algrenewcommand\algorithmicrequire{\textbf{Input:}}
\algrenewcommand\algorithmicensure{\textbf{Output}}

\makeatletter
\algnewcommand{\AlgComment}[1]{\Statex \hskip\ALG@thistlm {\color{gray}\textit{// #1}}}
\algnewcommand{\AlgCommentIndent}[1]{\Statex \hskip\ALG@tlm {\color{gray}\textit{// #1}}}
\makeatother

\newcommand\T{\boldsymbol{{T}}}
\newcommand\R{\boldsymbol{{R}}}
\newcommand\idd{\boldsymbol{{I}}}
\newcommand\xvec{\boldsymbol{x}}

\newcommand\tvec{\boldsymbol{t}}
\newcommand\gvec{\boldsymbol{g}}
\newcommand\rvec{\boldsymbol{r}}
\newcommand\xivec{\boldsymbol{\xi}}
\newcommand\omegavec{\boldsymbol{\omega}}
\newcommand\vvec{\boldsymbol{v}}
\newcommand\VF{\mathfrak{X}}
\newcommand\grad{\mathrm{grad}}
\newcommand\hess{\mathrm{Hess}}
\newcommand{\ad}{\mathrm{ad}}

\newcommand{\SE}{\mathsf{SE}}
\newcommand{\cov}{\boldsymbol{\Sigma}}
\newcommand{\D}{\mathrm{D}}
\newcommand{\zerovec}{\boldsymbol{0}}
\newcommand\manifold{\mathcal{G}\xspace}
\newcommand\liealgebra{\mathfrak{g}\xspace}
\newcommand{\tangentse}{T_{\T}\manifold}
\newcommand\ourmethod{\text{G-CVO}\xspace}
\newcommand\ellvec{\boldsymbol{\ell}}
\newcommand\zvec{\boldsymbol{z}}
\newcommand{\lidar}{\text{LiDAR}\xspace}

\usepackage{fancyhdr}

\definecolor{cvprblue}{rgb}{0.21,0.49,0.74}
\usepackage[breaklinks,colorlinks,allcolors=cvprblue]{hyperref}
\title{Generalized-CVO: Fast and Correspondence-Free Local Point Cloud Registration with Second Order Riemannian Optimization$^{1}$\thanks{$^{1}$Website: {\tt\scriptsize  \url{https://github.com/ToyotaResearchInstitute/gcvo.git}}}}

\author{
Ray Zhang$^{2}$, 
Marcus Greiff$^{2}$, 
Thomas Lew$^{2}$, 
John Subosits$^{2}$
\thanks{$^{2}$All authors are with Toyota Research Institute, 
4440 El Camino Real, Los Altos, California, USA. 
{\tt\small \{ray.zhang, marcus.greiff, thomas.lew, john.subosits\}@tri.global}}
}

\begin{document}
\maketitle
\thispagestyle{fancy}
\begin{abstract}
We propose a fast and correspondence-free local point cloud registration method 
that leverages geometric surface structure and reproducing kernel Hilbert space (RKHS) embeddings. The method represents point clouds as continuous functions with point-wise anisotropic kernels that encode local geometry. This formulation improves alignment along surface normals while relaxing alignment along tangential directions. To solve the resulting registration problem, we propose a second-order on-manifold optimization scheme with approximate Riemannian Hessians, achieving a speedup of up to 10x over the first-order solvers used in prior correspondence-free RKHS-based methods. We demonstrate improved frame-to-frame LiDAR and RGB-D tracking accuracy across diverse indoor and outdoor datasets. On a LiDAR tracking registration task in the driving domain, we achieve a reduction of $>55\%$ in both translational and rotational drift in challenging feature-sparse environments. On object registration benchmarks, we show improved robustness over ICP-based methods and further gains when refining global initialization, particularly under moderate misalignment.
%We evaluate object registration under moderate misalignment, demonstrating robustness to noise and outliers and gains from refining global initialization.
%\ray{We further evaluate the method on object registration benchmarks with moderate initial misalignment, where it demonstrates robustness to noise and outliers, and further improves alignment when used to refine initial estimates from global registration methods.} %On indoor datasets with large motion and less overlap in point clouds, the proposed method shows promising gains relative to prior correspondence-free methods.  
\end{abstract}
\section{Introduction}\label{sec:intro}

\begin{figure}[!t] 
\centering
\begin{tikzpicture}
  \node[anchor=north] at (0,1.0) {\includegraphics[width=0.8\columnwidth,trim=0 0 0 0,clip]{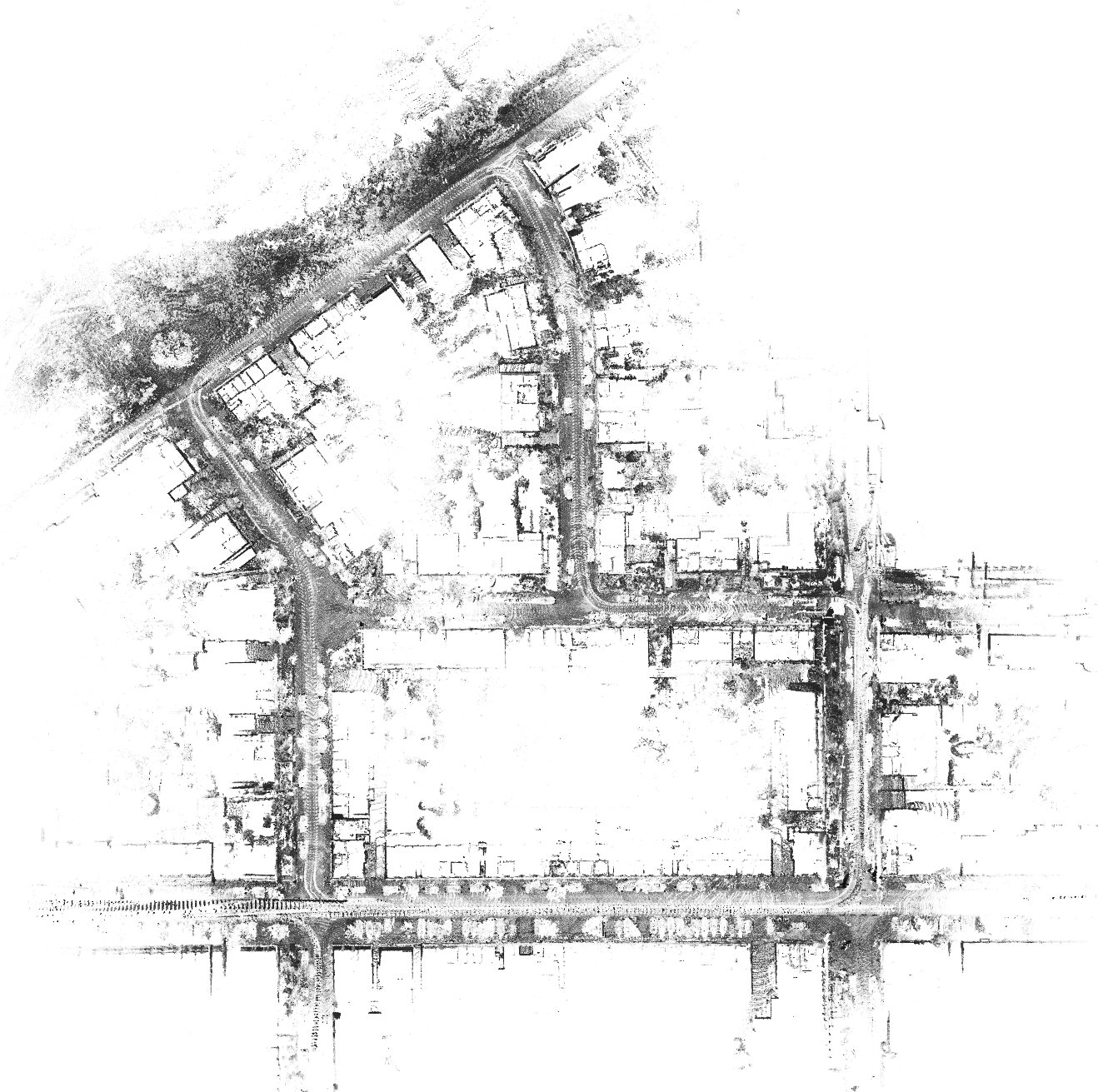}};
    
  \node[anchor=north] at (0,6) {\includegraphics[width=\columnwidth]{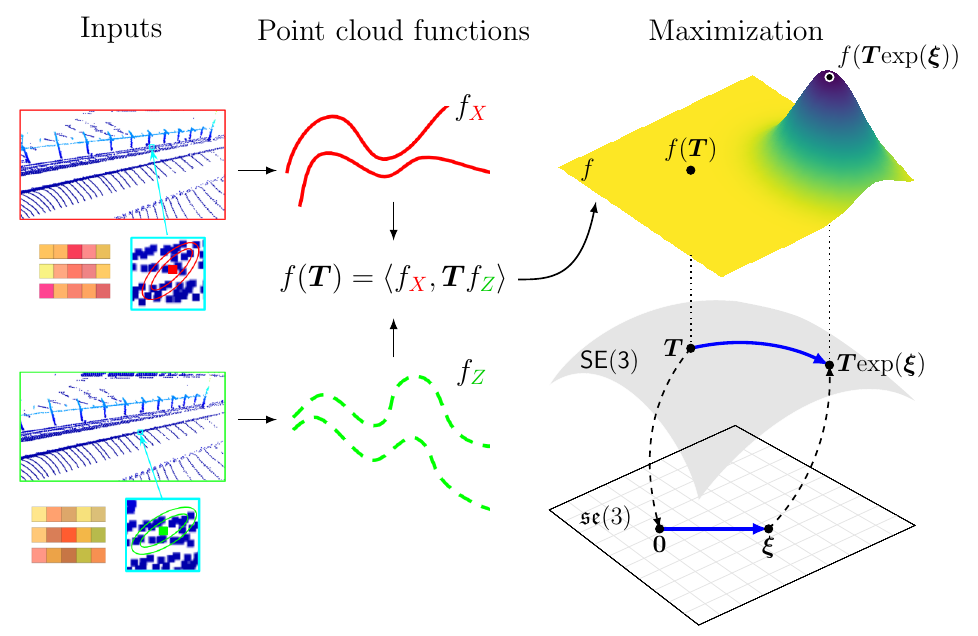}};

\draw[thin,-latex] 
  (3,0.7) .. controls (3,0.1) and (2.5,-0.4) .. (1.5,-1);

\end{tikzpicture}

\vspace*{-15pt}  

\caption{%The upper plot shows the formulation of the proposed correspondence-free point cloud registration
The proposed correspondence-free point cloud registration method in an RKHS using anisotropic kernel embeddings. The registration problem is solved locally via a second-order Riemannian Gauss-Newton optimization.
The lower plot shows a stacked point cloud map of KITTI~\cite{Geiger2012kitti} sequence \texttt{07}, constructed by frame-to-frame tracking without local mapping or loop closure.
}
\vspace*{-15pt} 
\label{fig:teaser}
\end{figure}

Point cloud registration is the process of aligning two or more 3D point clouds, i.e., sets of points in space that represent the geometric structure of an object or environment, so that they share a common coordinate system~\cite{besl1992method, segal2009gicp}. It has been widely adopted in driving~\cite{zhang2024lidar}, robotic autonomy~\cite{sun2020scalability}, navigation~\cite{kim2022forward_dynamics}, and scene reconstruction~\cite{vizzo2021poisson}, especially when other sensors such as GPS are unavailable or unreliable. In robot localization frameworks such as \lidar and visual odometry, registration is a key component that consists of estimating robot trajectories and spatially consistent world models using sequential point cloud measurements from \lidar and RGB-D  scans~\cite{lee2024lidar, cadena2016pastpresentslam}. %A typical \lidar odometry system includes a frontend that tracks incoming frames and a backend that manages the existing map~\cite{zhang2014loam}. 
{In odometry settings (often referred to as tracking), registration is commonly initialized with a pose prior that is iteratively refined with respect to a registration objective; this is the primary focus of this work. Such initialization can be obtained by integrating velocity or acceleration measurements~\cite{xu22fastlio2}. In contrast, global registration methods are designed to handle large initial errors~\cite{yang2015go,yang2020teaser} without utilizing pose priors.}

%The top figure illustrates the correspondence-free point cloud registration of two inputs in RKHS using anisotropic kernel embeddings. The second plot shows a stacked point cloud map of KITTI~\cite{Geiger2012kitti} sequence \texttt{07}, built from the estimated trajectory of frame-to-frame tracking with the proposed method, without local mapping or loop closure.
%\begin{figure}[t!]
%    \centering
%    \includegraphics[width=\columnwidth]{manifold.pdf}
%    \caption{Correspondence-free RKHS point cloud registration of ${\color{green!80!black}Z}$ to a target ${\color{red}X}$ using anisotropic kernel embeddings. }
 %   \label{fig:placeholder}
%\end{figure}

Widely adopted in %recent 
vision-based odometry systems, correspondence-based point cloud registrations may perform well in feature-rich environments~\cite{zhang2014loam, shan2018lego, behley2018suma, oelsch2021rloam, wang2021floam,  koide2021gicpcost, vizzo2023kiss, zheng2024traj}, but face challenges in feature-sparse settings~\cite{Zhao_2024_CVPR_subt}. 
%These methods usually follow a two-step alternating optimization scheme. First, for each point in the target point cloud, the nearest neighbors in the source point cloud are identified. The distances between these pairs are aggregated, and the total sum of the residuals is minimized over the relative pose between the clouds~\cite{besl1992method, segal2009gicp}. 
Correspondence-based registrations typically use a two-step alternating optimization procedure~\cite{censi2008icp}, first identifying nearest neighbors between target and source point clouds (correspondence), followed by finding the relative pose by minimizing aggregated residuals subject to the correspondences~\cite{besl1992method, segal2009gicp}. In the correspondence search, the matching process often leverages a spectrum of feature representations, ranging from efficient planar features~\cite{segal2009gicp} and bespoke local descriptors~\cite{rusu2009fpfh} to computationally expensive deep learning–based feature embeddings~\cite{choy2019fcgf, huang2021predator, qin2023geotransformer}.
%Correspondence-based registrations form the foundation of many \lidar-only and \lidar-IMU odometry systems~\cite{zhang2014loam, xu22fastlio2, %}, frequently used in driving~\cite{zhang2024lidar}. 
However, correspondence-based methods rely on the key assumption that point associations can be established reliably, which may not hold in feature-sparse environments such as rural or off-road navigation~\cite{Zhao_2024_CVPR_subt}. Robustness challenges in such environments have motivated the use of robust loss functions~\cite{huber1992robust, aftab15irls} and certifiable  methods~\cite{yang2020teaser}. 

\iffalse
\begin{figure}[t!]
    \centering
    \begin{tikzpicture}
    \node[] at (0,0) %{\includegraphics[width=\columnwidth]{icra_fig_no_labels.png}};
    {\includegraphics[width=0.9\columnwidth]{hero_updated.png}};
    \node[] at (2.8,0) {
    $\underset{\T\in\mathsf{SE(3)}}{\max} \langle f_{\color{red}X}, \T  f_{{\color{green!80!black}Z}}\rangle$
    };
    \node[fill=white] at (-2.85,2.33) {{\color{red}Target: $X$}};
    \node[fill=white] at (-2.35,-2.35) {{\color{green!80!black}Measurement: $Z$}};
    \node[fill=white] at (3.15,1.34) {{\color{red} $\boldsymbol{\Sigma}(\xvec)$}};
    \node[fill=white] at (3.15,-2.30) {{\color{green!80!black} $\boldsymbol{\Sigma}(\zvec)$}};
    \end{tikzpicture}
    \vspace*{-10pt}
    
    \caption{Correspondence-free RKHS point cloud registration of ${\color{green!80!black}Z}$ to a target ${\color{red}X}$ using anisotropic kernel embeddings. }
    \label{fig:placeholder}
\end{figure}
\fi

To tackle challenges in feature-sparse and noisy environments, correspondence-free registration methods such as Continuous Visual Odometry (CVO)~\cite{MGhaffari-RSS-19} have emerged as a promising alternative. In contrast to discrete point sets, correspondence-free methods represent point clouds as continuous functions embedded in a reproducing kernel Hilbert space  
(RKHS)~\cite{clarkmaani20, MGhaffari-RSS-19, zhang2024correspondence,  zhang2025rkhs}. This line of work demonstrates robustness to sensor noise and outliers, as it does not rely on known pairwise correspondences~\cite{zhang2025rkhs}. Additionally, it enables the integration of pose-invariant features such as color and pixel labels through the functional representation~\cite{zhang2021new}. 
However, existing RKHS-based formulations do not exploit local geometric structures often present in point cloud measurements, including surfaces and edges~\cite{segal2009gicp, zhang2014loam}. Furthermore, the cited RKHS registrations rely on  1st-order solver implementations, which are prohibitively slow in applications such as high-performance driving~\cite{lew2024risk,djeumou2024one}.%, where tracking and controllers must run at high rates. 

\textbf{Contributions}: To address these limitations, we propose a correspondence-free objective in an RKHS encoding local geometric information and a second-order on-manifold solver, \textbf{G}eneralized \textbf{CVO} (\ourmethod). %Each point cloud is represented as a function in RKHS with anisotropic local kernel embeddings. 
Our RKHS kernel embeddings are implicit and do not rely on additional feature extraction, allowing the registration to fit tightly along surface normals while being less constrained along the surface tangent plane, and resulting in improved robustness and performance in standard registration and tracking metrics~\cite{Geiger2012kitti}.  %While we focus on \lidar and RGB-D frame-to-frame tracking, we also show that the proposed method can be used for larger-motion registrations under less overlap.
%\subsection{Contributions}
Specifically, we make the following contributions:
\begin{itemize}
    \item We extend the RKHS loss function in~\cite{MGhaffari-RSS-19} to encode local geometries with anisotropic kernel embeddings.
    \item We propose a second-order solver using Riemannian Gauss-Newton that is significantly faster than the first-order optimizers in the cited RKHS registration methods.
    \item Results show that \ourmethod achieves superior tracking accuracy and robustness on the urban driving datasets of KITTI~\cite{Geiger2012kitti}, self-collected feature-sparse vehicle-racing datasets, and the indoor ETH RGB-D dataset~\cite{schops2019badslam}. Furthermore, \ourmethod shows robust ModelNet40~\cite{zeng20173dmatch} registration under moderate misalignment.
    %, with improvements from global initialization. 
\end{itemize}
%Additionally, we show that the approximated Riemannian Gauss-Newton solver improves computation times by up to one order of magnitude, which can be used to speed up other RKHS-based registration methods, notably~\cite{MGhaffari-RSS-19,zhang2021new}.

\section{Related Works}\label{sec:prelim}
%\subsection{Frame-to-Frame Registration}\label{sec:frametoframe}

ICP-based methods are widely used for frame-to-frame registration and consist of an alternating optimization procedure~\cite{besl1992method, chen1992object, segal2009gicp, koide2021voxelized, LiCVPR16edges}.  
%These methods usually follow a two-step alternating optimization scheme. First, for each point in the target point cloud, the nearest neighbors in the source point cloud are identified. The distances between these pairs are aggregated, and the total sum of the residuals is minimized over the relative pose between the clouds
The first step finds point correspondences, usually with a KD tree~\cite{greenspan2003approximate, cai2021ikdtree}. The second step uses an  SVD~\cite{besl1992method} or an on-manifold solver~\cite{bhattacharya2019efficient, breux2022manifold} to find an optimal transformation subject to the correspondences. %This method has been successfully applied to smoothing, mapping, and \lidar-IMU odometry~\cite{zhang2014loam, shan2018lego, behley2018suma, oelsch2021rloam, wang2021floam,  koide2021gicpcost, vizzo2023kiss, zheng2024traj}. 
% Subsequent works have extended ICP to more general manifolds by convex optimization~\cite{yang2020teaser}.
To address the error-prone nature of finding point correspondences in ICP, robust estimation techniques such as M-estimators are often employed to mitigate the impact of large residuals~\cite{stefanski2002calculus}. The robust loss functions can be optimized by graduated non-convexity (GNC)~\cite{yang2020teaser} or iteratively reweighted least squares (IRLS)~\cite{aftab15irls}. Unlike the methods above, our approach does not require pairwise correspondences, which may reduce sensitivity to outliers~\cite{zhang2024correspondence}.  

\begin{comment}
\subsection{LiDAR Smoothing}
LiDAR smoothing and mapping extend the registration of two sequential frames into a registration between an online-constructed map and the %
\begin{itemize}
    \item LOAM, IMLS-SLAM,  LeGo-LOAM, Suma, DLO, F-LOAM, A-LOAM, MULLS, NDT-LOAM, R-LOAM, GenZ-ICP
    \item Continuous time methods: CT-ICP, Kiss-ICP, Traj-LO
    \item Differentiable map representations: PINS-SLAM, 
\end{itemize}
\end{comment}

%\subsection{Registration with local geometric information}
Another approach to improve robustness is to leverage the local structural information in point clouds. ICP residuals can be reformulated as point-to-feature residuals, such as point-to-plane~\cite{chen1992object, segal2009gicp} or point-to-edge~\cite{censi2008icp}, which have shown significant improvements in practical \lidar odometry applications~\cite{zhang2014loam, deschaud2018imls,xu2021fastliofastrobustlidarinertial,xu22fastlio2, azzini2025balo}. %  ICP, Point-to-Plane ICP, TrICP, GICP, GICP-SE(3), Fast-VGICPto-edge 
%NDT
Some probabilistic approaches model point clouds as Gaussian mixture models (GMM)~\cite{bishop2006prml, eckart2018hgmr}, encoding surfaces and edges by Gaussian distributions. NDT~\cite{magnusson2009three, das2014scan, miao2016pose} discretizes space and fits Gaussians to each voxel, and then fits a second point cloud to the voxelized collection of Gaussians. %\color{red}While promising, these methods still rely on accurate point correspondences, making them liable to outliers. 
Neural networks have been adopted to embed point clouds to capture local invariant and equivariant features in the feature space~\cite{phan2018dgcnn, yu2021cofinet,  huang2021predator, zhu2022so3,  
zhu2023e2pn, qin2023geotransformer,  lin2024se3et}. Learning-based methods provide rich feature embeddings and coarse correspondences, even under large motions, but a fine-grained registration stage is typically required, either to drive the alignment with these features~\cite{zhou2016fast,  yang2020teaser, Chen_2022_sc2pcr, yang2024mac,  zhao2025regent} or to refine it as post-processing~\cite{segal2009gicp,  yang2024mac, zhao2025regent}. Our method can incorporate such features in the kernel embeddings, or omit feature extraction altogether to favor computational efficiency.

%In contrast, our method can either integrate pose-invariant features directly into the functional representation or dispense with feature extraction entirely when fast operation is required.

%The other way around is to minimize the divergence of two point clouds' GMM distributions~\cite{wang2006groupwise}.

%Recently, deep learning-based approaches learn locally invariant or equivariant features.
%\subsection{RKHS-Based methods}\label{sec:RKHSrelatedwork}
RKHS-based registration addresses the limitations of ICP by removing the need for correspondences~\cite{williams2006gaussian}. This family of methods includes CVO~\cite{MGhaffari-RSS-19} and SemanticCVO~\cite{zhang2021new}, which represent color and geometric information as functions in RKHS and formulate the registration problem as the maximization of the inner product between two point cloud functions (see Fig.~\ref{fig:teaser}). Kernel-Correlation~\cite{tsin2004correlation} is a special case of CVO that does not use pixel labels. RKHS-BA~\cite{zhang2025rkhs}  extends frame-to-frame registration to multi-view bundle adjustment, enabling large-scale mapping. Equiv-Align~\cite{zhang2024correspondence} makes RKHS registration differentiable, enabling end-to-end learning of point-wise features. 
Unlike CVO, which uses an isotropic kernel, \ourmethod constructs spatially varying anisotropic kernels that adapt to local geometry, improving registration accuracy.

%Additionally, \ourmethod is amenable to parallelization, as the Riemannian gradients and Hessians can be expressed in closed form.

%The proposed method in this work builds upon the CVO method in~\cite{MGhaffari-RSS-19,Zhang2020semanticcvo}, and extends these formulations with additional constraints imposed through the kernels of the RHKS from surface scans.

%The method proposed in this paper similarly enforces geometric constraints as in point-to-plane losses, but improves robustness by using convergence in the norm of functions rather than relying on point-wise correspondences.
 % -*- TeX-master: "root.tex" -*-

\subsection{Notation}
Let $X=\{(\xvec_i,\ellvec_i^{X})\in\mathbb{R}^3\times \mathbb{R}^{l}\;:\;i\in[N_X]\}$ be a target point cloud or map, and $Z=\{(\zvec_i,\ellvec_i^{Z})\in\mathbb{R}^3\times \mathbb{R}^{l}\;:\;i\in[N_Z]\}$ be a new point cloud  to be registered, where $\ellvec_i^{j}$ denotes pose-invariant features such as intensity. $N_X$ and $N_Z$ may differ. We denote $\bar{X}$ and $\bar{Z}$ as sets that only contain the 3D coordinates. A pose relating these frames is denoted $\T\in\mathsf{SE(3)}$. Its tangent space is defined as $T_{\T}\SE(3)=\{\vvec: \vvec=\boldsymbol{\gamma}^{\prime}(0), \boldsymbol{\gamma}(0)=\T\}$ where $\boldsymbol{\gamma}:\mathbb{R}\rightarrow \SE(3)$ is a smooth curve passing through $\T$. In particular, its Lie algebra is the tangent space at the identity, i.e. $\mathfrak{se}(3)=T_{\idd}\SE(3)$. In the following we use the shorthand $\manifold=\mathsf{SE(3)}$ with $\liealgebra=\mathfrak{se}(3)$. % denoting the Lie algebra. 

The action of a pose in $\manifold$ on the 3D points is  $\T \xvec=\R \xvec  + \tvec$, where $\R\in \mathsf{SO(3)}$ and $\tvec\in \mathbb{R}^3$. We let $(\cdot)^{\land}:\mathbb{R}^6\mapsto \liealgebra$ 
denote the \emph{hat} map, mapping a twist vector $\xivec=(\omegavec^{{\top}}, \vvec^{{\top}})^{\top}$ with rotational and positional components $\omegavec\in \mathbb{R}^3$ and $\vvec\in \mathbb{R}^3$, respectively, to the Lie algebra $\liealgebra$. We let the \emph{vee} map $(\cdot)^\vee$ denote the inverse of the hat map. %, and use it for the  represention with a 6D vector via pulling back to $\liealgebra$ by left translation. 
With a slight abuse of notation, we denote the vee map of $\mathsf{SO}(3)$  by $(\cdot)^{\vee}:\mathfrak{so}(3)\mapsto\mathbb{R}^3$. Then a twist acts on a point $\xvec\in\mathbb{R}^3$ by $\xivec^{\wedge} \xvec = \omegavec^{\wedge} \xvec + \vvec$.
\iffalse, and
\begin{align*}
    \xvec^{\wedge} = \begin{bmatrix}
        x\\
        y\\
        z
    \end{bmatrix}^{\wedge} \hspace{-4pt}= 
    \begin{bmatrix}
        0 & -z & y\\
        z & 0 &  -x\\
        -y & x & 0
    \end{bmatrix}\quad\text{if}\quad\xvec\in\mathbb{R}^3.%, \begin{bmatrix}
    %    \omegavec\\
    %    \vvec   
    %\end{bmatrix}^{\wedge} \hspace{-4pt}= 
    %\begin{bmatrix}
    %    \omegavec^{\wedge} & \vvec\\
    %     \zerovec & 0
    %\end{bmatrix}\in \liealgebra
\end{align*}
\fi
The matrix exponential $\exp: \liealgebra\mapsto \manifold$ maps a twist to a pose, and the  logarithm defines its inverse:
\begin{align}
    \exp(\xivec^{\wedge})&=\T,&
    \log(\T)&= \xivec^{\wedge}.
\end{align}
A vector field $\VF(\manifold)$ assigns one tangent element $U\in \tangentse$ to each $\T\in \manifold$. To measure distances on the manifold, we consider a Riemannian metric~\cite{boumal2023optmfd}, which is a smooth map $\langle \cdot, \cdot \rangle: \tangentse\times \tangentse\mapsto \mathbb{R}$. A Riemannian metric is left-invariant if $\langle U, V\rangle=\langle \T^{-1}U, \T^{-1}V\rangle$, which enables pulling back the metric on $\tangentse$ at $\T$ to the metric on $\liealgebra$ at the identity element.
We use a specific left-invariant metric 
\begin{align}
  \langle ((\omegavec_1^{\top}, \vvec_1^{\top})^{\top})^{\wedge}, ((\omegavec_2^{\top}, \vvec_2^{\top})^{\top})^{\wedge} \rangle = \omegavec_1^{\top}\omegavec_2 + \vvec_1^{\top}\vvec_2.
  \label{eq:se3_metric}
\end{align}
Throughout, tangent quantities on $\mathcal{G}$ are left-trivialized by left translation with $\T^{-1}$ onto $\liealgebra \simeq \mathbb{R}^6$ via the vee map$(\cdot)^\vee$.

\subsection{Riemannian Gradients and Hessians}
To derive Riemannian gradients, we note that for a smooth function $f: \manifold\rightarrow \mathbb{R}$, and any smooth curve $\boldsymbol{\gamma}: \mathbb{R}\rightarrow \manifold$ satisfying $\boldsymbol{\gamma}(0)=\T\in\manifold$ with $\boldsymbol{\gamma}'(0) = U\in \tangentse$,  the gradient relates to the directional derivative along $U$ by 
\begin{align}\label{eq:equivalence}
    \mathrm{D}f(\T)[U]=\langle \grad f(\T), U\rangle=\frac{\mathrm{d}}{\mathrm{d}t}\Big|_{t=0}f(\boldsymbol{\gamma}(t)).
\end{align}
We refer to~\cite{boumal2023optmfd} for details, and use the equivalence in~\eqref{eq:equivalence} when deriving the Riemannian gradients $\grad f$ in Sec.~\ref{sec:gencvo}.

To derive Hessians, we need to take derivatives with respect to Riemannian gradients. When calculating the directional derivative of $\grad f$, directly subtracting two tangent vectors is not well-defined when they reside in different tangent spaces. Instead, we use the Levi-Civita connection to project the Hessian onto the tangent space. We denote the connection by $
  \nabla: \VF(\manifold)\times\VF(\manifold) \mapsto \VF(\manifold)$ (see~\protect{\cite[Chapter 5.3]{absil2008optimization} for details), 
and with it, the Riemannian Hessian is 
\begin{align}\label{eq:hessequiv}
\hess f(\T)[U]=\nabla_{U} \grad f({\T}) . 
\end{align}
Following~\cite{absil2008optimization}, %and using the metric in~\eqref{eq:se3_metric}, 
 we  express~\eqref{eq:hessequiv} in the context of $\manifold$ as
\begin{align}
  \nabla_{U} \grad f = \mathrm{D}(\grad f)[U] + \Gamma(U, \grad f).
  \label{eq:koszul}
\end{align}
The first term is the directional derivative of the Riemannian gradient $\grad f$ along $U$. The  correction term $\Gamma(U, \grad f)$ in~\eqref{eq:koszul} can be derived from the  Koszul formula~\protect{\cite[Chapter 21.3]{gallier2020differential} with the left-invariant metric in~\eqref{eq:se3_metric} for $\liealgebra$: %, and can be expressed compactly as 
\begin{align}
%   \Gamma(U, \grad f) \hspace{-1pt}=\hspace{-1pt} \tfrac{1}{2}([U, \grad f] \hspace{-1pt}&-\hspace{-1pt} (\mathrm{ad}_{U}^{*} (\grad f)^{\vee})^{\wedge} \hspace{-1pt}
%   \\
%   &-\hspace{-1pt} (\mathrm{ad}_{\grad f}^{*} (U)^{\vee} )^{\wedge}),
   \Gamma(U,&\grad f) \hspace{-2pt}=\hspace{-2pt} \tfrac{1}{2}([U, \hspace{-2pt}\grad f] \hspace{-2pt}-\hspace{-2pt} (\mathrm{ad}_{U}^{*} \grad f^{\vee}\hspace{-1pt})^{\wedge} 
   \hspace{-2pt}-\hspace{-2pt} (\mathrm{ad}_{\grad f}^{*} U^{\vee}\hspace{-1pt})^{\wedge}\hspace{-1pt}),
   \label{eq:gamma_se3}
   \vspace{-10pt}
\end{align}
where $[\cdot, \cdot]:\liealgebra\times \liealgebra\rightarrow \liealgebra$ is the Lie bracket and the adjoint is
\begin{align}\label{eq:adjoints}
\mathrm{ad}_{\xivec^{\wedge}}:=\begin{bmatrix}
    \omegavec^{\wedge} & \zerovec\\
    \vvec^{\wedge} & \omegavec^{\wedge}
  \end{bmatrix}, \quad \mathrm{ad}_{\xivec^{\wedge}}^{*}:=\begin{bmatrix}
    -\omegavec^{\wedge} & -\vvec^{\wedge}\\
    \zerovec & -\omegavec^{\wedge}
  \end{bmatrix}.%,
\end{align}
Let $\gvec(\T)\triangleq(\grad f(\T))^{\vee}$ 
 denote the 6D vector representation of $\grad f(\T)$. The Hessian  in the vector form is 
 \begin{align}
   \hspace{-2pt}  (\hess f(\T)[U])^{\vee}\hspace{-2pt}=\hspace{-2pt}\mathrm{D}\gvec(\T)[U^{\vee}]\hspace{-2pt} + \hspace{-2pt}\Gamma^{\vee}(U,\hspace{-1pt} \grad f(\T)).
     \label{eq:hess_se3_vec}
 \end{align}

\subsection{Classical RKHS loss functions}\label{sec:rkhsloss}
We briefly introduce point cloud registrations in RKHS from CVO~\cite{MGhaffari-RSS-19, zhang2021new}. In CVO's formulation, point clouds are represented as continuous functions in an RKHS $\mathcal{H}$~\cite{MGhaffari-RSS-19} instead of discrete point sets. For a point cloud $X$, we let
\begin{align}
\label{eq:point_cloud_function}
f_X(\cdot)=\hspace{-5pt}\sum_{(\xvec, \ellvec^{X})\in X} \hspace{-5pt}c(\ellvec^X) k(\xvec, \cdot),
\end{align}
 where $c: \mathbb{R}^{l} \mapsto \mathbb{R}^{+}$ is a   %scalar-valued intensity 
 label function of pose-invariant features, such as \lidar points' intensities and RGB-D points' colors. Here, $k(\cdot): \mathbb{R}^3\times \mathbb{R}^3\mapsto \mathbb{R}$ is a kernel function.
 In the following, we denote $\T f_X(\cdot)=f_{\T^{-1} X}( \cdot)$ for $\T\in\manifold$ as the pose $\T$ acting on the representation $f_X$ of the frame $X$. The distance in the RKHS  between the functions associated with a source frame $X$ and a target frame $Z$  is 
\begin{align}
\label{eq:rkhs_dist}
\hspace{-1pt}\|f_X-\T f_Z&\|_{\mathcal{H}}^2\hspace{-1pt}=\hspace{-1pt}\langle f_X, f_X\rangle \hspace{-1pt}+\hspace{-1pt} \langle \T f_Z, \T f_Z \rangle \hspace{-1pt}-\hspace{-1pt} 2\langle f_X, \T f_Z\rangle\notag\\&\hspace{-1pt}=\hspace{-1pt}\langle f_X, f_X\rangle \hspace{-1pt}+\hspace{-1pt} \langle f_Z,   f_Z \rangle \hspace{-1pt}-\hspace{-1pt} 2\langle f_X, \T  f_Z\rangle,\hspace{-1pt}
\end{align}
%\begin{remark}
%    For an inner product space, an isometry is a map $h$ such that $\langle hx, hy \rangle = \langle x, y\rangle$.
%\end{remark}
%This \lidar point cloud registration (\PCR) problem is typically formulated as a constrained optimization problem
%\begin{equation}\label{PCR}
%\min_{\Tcal} \sum_{\substack{{x_k^i\in X_i}\\{x_p^j\in X_j}\\ i\neq j}}
%d(x_k^i,T_i^jx_p^j),\tag{\PCR}
%\min_{\T\in\mathsf{SE(3)}}  d(X,\T Z),\tag{\PCR}
%\end{equation}
%where $d$ is some distance measure between the representations of the two point clouds. In ICP-based methods, $d$ can be chosen as Euclidean or Mahalanobis distances between  selected closest point pairs. 
%where %the equality is due to $\T\in\manifold$ is an isometry. 
To minimize the distance between the functions $f_X$ and $f_Z$, we maximize the third term containing both point clouds and $\T$. This approach is akin to ICP's objective function, but posed as an optimization problem that no longer requires pair-wise correspondences: 
\begin{align}
%\label{PCRH} 
\max_{\T\in\manifold} \ \ \langle f_X, \T  f_Z\rangle.\label{eq:PCRH}
\end{align}
Using the kernel trick~\cite{zhang2021new}, this loss can be simplified  as %by~\eqref{eq:point_cloud_function}
%\begin{subequations}
\begin{align}\label{eq:rkhs_loss}
    \langle f_X, \T  f_Z\rangle  
    &=\hspace{-7pt} \sum_{\substack{{(\xvec_i, \ellvec^{X})\in X}\\ {(\zvec_j, \ellvec^{Z})\in Z}}} \hspace{-7pt}\langle \ellvec^{X},\ellvec^{Z}\rangle k(\xvec_i, \T^{-1} \zvec_j) \triangleq f(\T),
    \raisetag{15pt}
\end{align} 
%\end{subequations}
with a constant $\langle \ellvec^{X},\ellvec^{Z}\rangle$  encoding the intensity similarity of \lidar points. The isotropic squared exponential kernel is commonly used~\cite{clarkmaani20}, but we have freedom in choosing the kernel embedding, which we will discuss in Sec.~\ref{sec:gencvo}. The classical RKHS-based registrations rely on implementations of 1st-order Riemannian gradient ascents. %This down-weights point pairs that are geometrically close but have different intensities.

%When evaluating the RKHS methods, we consider a full trajectory by chaining the relative poses from pairwise registrations.  In frame-to-frame tracking, $X$ is the point cloud of the last frame, while $Z$ is the incoming scan. Modern \lidar odometry often adopts model-to-frame tracking where $X$ comes from a local dense point cloud  map~\cite{zhang2014loam, xu22fastlio2, zheng2024traj} and $Z$ is the latest frame. This work is entirely compatible with the latter, but demonstrated in frame-to-frame tracking as this more clearly elucidates differences in robustness of the registration problem.

\section{Generalized CVO}\label{sec:gencvo}
To improve  convergence  compared to  prior work discussed in Sec.~\ref{sec:prelim}, we propose the %\textbf{G}eneralized \textbf{CVO} (
\ourmethod registration method. It consists of two main components: 1) an anisotropic kernel embedding, inspired by methods that encode local surface structure, and 2) iterative optimizers that use closed-form expressions of the Riemannian gradients and Hessians. %subject to this kernel embedding.  
Both the 1st- and 2nd-order versions of \ourmethod are sketched in Alg.~\ref{alg:IF}.

\subsection{Anisotropic covariance}
In RKHS-based methods, we have significant freedom in choosing the kernel embedding. Previous formulations proposed the isotropic kernels~\cite{williams2006gaussian} due to their simplicity and smoothness. Observing that point cloud scans are usually surface scans, %, and that tt has improved registration performance in prior ICP-based works~\cite{segal2009gicp, zhang2014loam, koide2021voxelized}, 
we generalize the isotropic kernels in~\cite{MGhaffari-RSS-19} to anisotropic exponential kernels that model local surfaces
\newcommand\yvec{\boldsymbol{y}}
\begin{subequations}\label{eq:kernel}
\begin{align}
    k(\xvec, \zvec)
    &=\sigma^2 \exp(-\dfrac{1}{2}\langle (\xvec - \zvec), \boldsymbol{\cov}(\xvec,\zvec)^{-1} (\xvec - \zvec)\rangle),\notag\\
    \boldsymbol{\cov}(\xvec; \bar X)%^{(0)} 
     &= \dfrac{1}{n-1}\sum_{\yvec \in \mathcal{N}_{\bar X}(\xvec)} (\yvec - \xvec)(\yvec - \xvec)^\top,\\
     \boldsymbol{\cov}(\xvec,\zvec)%^{(0)} 
     &= \boldsymbol{\cov}(\xvec; \bar X) + \R^{\top} \boldsymbol{\cov}(\zvec; \bar Z) \R,
\end{align}\label{eq:covariance}
\end{subequations}
where $\mathcal{N}_{\bar X}(\xvec)$ are the $n$ nearest points in $\bar X$ around  $\xvec$, and $\R$ is the rotation matrix between $X$ and $Z$. In the following, we let $\cov_{ij}\triangleq \cov(\xvec_i, \zvec_j)$, and (i) do a KD-tree ball-based lookup at $\xvec$, (ii) compute the distances to each point in this ball, and (iii) select the nearest $n$ points to constitute this set. 
 This approach generalizes the separate edge and surface loss functions in some ICP-based methods without dedicated feature extractions~\cite{zhang2014loam, shan2018lego}, as shown in Fig.~\ref{fig:kernel_illustration}. %The anisotropic kernel smooths out the points lying on the surfaces and is evaluated over all of the points, without requiring dedicated feature extraction, as in~\cite{zhang2014loam, shan2018lego}.

\begin{remark} 
%Using a sample covariance results in sparse loss functions when bootstrapping the optimization, as the kernel will annihilate the majority of the points in the normal direction before they contribute to the loss. In practice, we bound the largest eigenvalues of $\cov(\xvec,\zvec)$ to constrain the weight ratios between normal and tangential directions.
When the empirical covariance is used in~\eqref{eq:covariance}, the induced anisotropy can render the loss function both sparse and noisy: the kernel attenuates components aligned with high-variance normal directions, which can be detrimental to registration performance. To regularize this degeneracy, we impose upper bounds on the dominant eigenmodes of $\cov(\xvec,\zvec)$, constraining the normal–tangential weighting ratio and ensuring a numerically stable contribution of all geometric directions to the loss.
\end{remark}

%scale the covariance by a scalar $\ell$
%\begin{align}
%\label{eq:kernel_scale}
%    \ell^{(n+1)}&=\ell^{(0)}{\alpha}^{n} \\
%    \boldsymbol{\cov}(\xvec)^{(n+1)} &= \boldsymbol{\cov}(\xvec)^{(0)}\ell^{(n+1)}
%\end{align}
%and 
%gradually shrink the scale $\ell$ by a percentage $\alpha$ during the iterative optimization (see Alg.~\ref{alg:IF}). The scale decay scheme takes place whenever the objective~\eqref{eq:rkhs_loss} converges at the current scale.

%% Hey Ray! I would add this scalar explicitly in (7) and then reference it from Alg 1.

%There are many choices of the kernel function. Previous RKHS-based formulations choose the isotropic square exponential kernels~\cite{williams2006gaussian} because of its simplicity and smoothness.  Observing that point cloud scans are usually surface scans, we use an anisotropic exponential kernel centered at $\xvec_k^i$ to resemble the surface structure of each position of the point clouds:
%\newcommand\yvec{\boldsymbol{y}}
%where $\yvec^i$ are the $n$ neighborhood points around $\xvec_k^i$. In practice, we (i) do a KD-tree ball-based lookup, (ii) compute the distances to each point in this ball, and (iii) select the nearest $n$ points and use these to constitute $\mathcal{N}$. 

%This is a generalization of the seperate edge and surface loss functions in a lot of ICP-based methods, as shown in~\ref{fig:kernel_illustration}. The anisotropic kernel effectively spatially smooth the points lying on the surfaces, while segment out line points as well.

\begin{figure}[t!]
\includegraphics[trim=0cm 0cm 0cm 0cm, clip=true, angle=0, width=0.49\columnwidth]{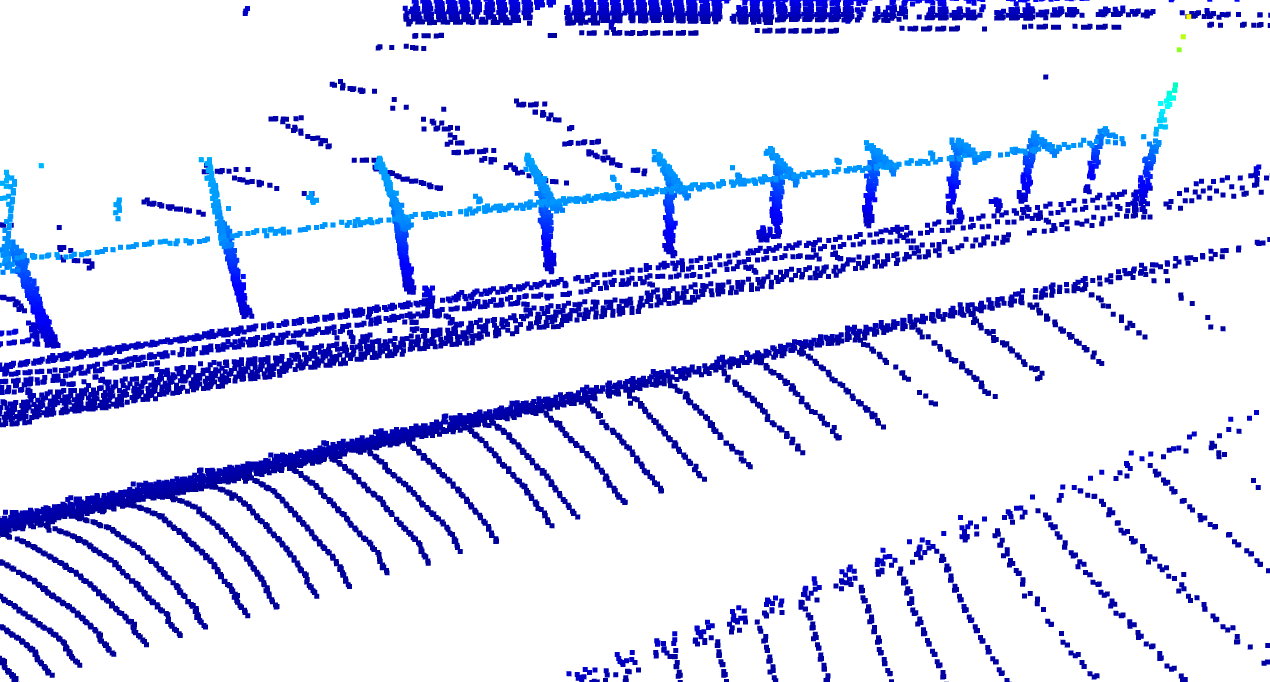}
\includegraphics[trim=0cm 0cm 0cm 0cm, clip=true, angle=0, width=0.49\columnwidth]{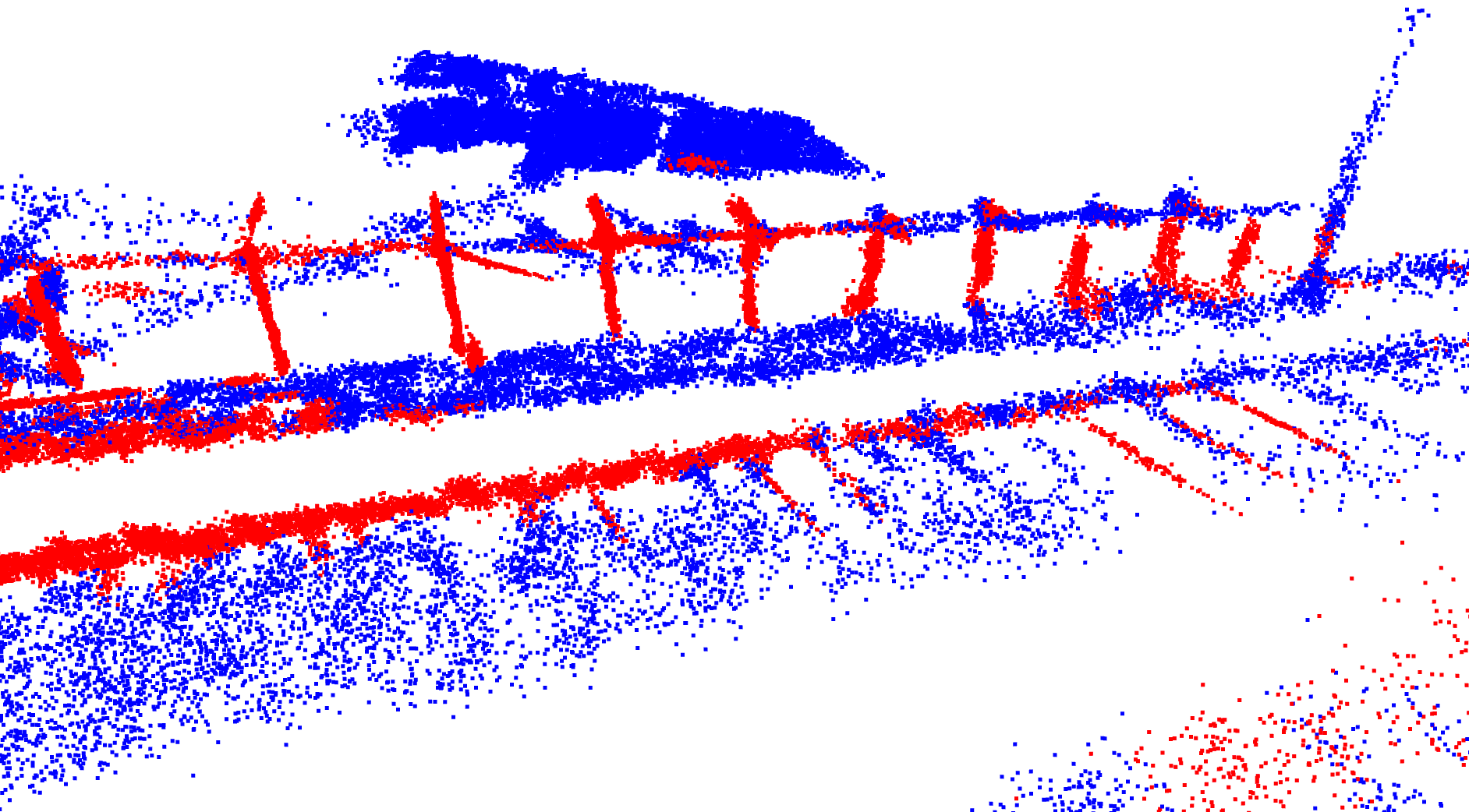}
\caption{Raw point cloud $Z$ (left), and resampled point cloud (right) in red if $\cov(\zvec_i; \bar{Z})$ has two small eigenvalues (edge) and blue if $\cov(\zvec_i; \bar{Z})$ has one small eigenvalue (surface).}
\vspace*{-10pt} 
\label{fig:kernel_illustration}
\end{figure}

%\begin{remark}
%The covariance matrix later has two indices, $k$ and $p$. Here $k\lor p$ if both indices are present. 
%\end{remark}
%\ray{TODO: add a figure illustrating the shape of the edge and surface kernels}

%\subsection{Gradient Ascent with Riemannian Gradients}

%To solve~\eqref{eq:rkhs_loss}, we derive the analytical Riemannian gradient.s %{\color{red}and Hessian}. We investigate two ways of solving it, via first-order gradient ascent and via {\color{red}Newton's method}.

\subsection{Calculating Riemannian gradients}
In a first variant of \ourmethod, we follow~\cite{MGhaffari-RSS-19} and solve \eqref{eq:PCRH} by gradient ascent (hereinafter \ourmethod-1). As the objective is a sum of exponentials on $\manifold$, taking gradients in a Euclidean space disregards the structure of $\manifold$ and can result in suboptimal linearizations~\cite{Grisetti10}. 
\iffalse
\begin{definition}
    For each point $x\in M$, a Riemannian metric $\langle \cdot, \cdot \rangle_x: T_xM\times T_xM\rightarrow \mathbb{R}$ is a map from the tangent space of $M$ at $x$ that varies smoothly with x. For any tangent vector $X, Y\in T_xM$,  the map $x\rightarrow \langle X(x), Y(x)\rangle_M$ is a smooth function on $M$~\cite{boumal2023optmfd}.
\end{definition}

\begin{remark}
    For a smooth function $f: M\rightarrow \mathbb{R}$, for any smooth curve $\gamma: \mathbb{R}\rightarrow M$ passed through a point $\gamma(0)=x$ on $M$ with $\gamma'(0) = v$,  the derivative along direction $v\in T_xM$ is $Df(x)[v]=\dfrac{d}{dt}\bigg|_{t=0}f'(\gamma(t))$~\cite{boumal2023optmfd}.
\end{remark}
\begin{remark}
    A metric on a manifold $M$ is a choice of inner product at $x$ for any $x\in M$, where the inner product $\langle \cdot, \cdot \rangle_x: T_xM\times T_xM\rightarrow \mathbb{R}$ is a symmetric, bilinear and positive definite map from the tangent space of $M$ at $x$ to $\mathbb{R}$~\cite{boumal2023optmfd}.
\end{remark}
\fi
Instead, we take the directional derivative of~\eqref{eq:rkhs_loss} at $\T$  along a perturbation $\xivec^{\wedge}$ by~\eqref{eq:equivalence} to derive the Riemannian gradients with the metric~\eqref{eq:se3_metric}.

\begin{lemma}\label{lem:gradient}
    The Riemannian gradient of~\eqref{eq:rkhs_loss} at $\T\in\manifold$ is $\grad f(\T)$ with $ (\grad f(\T))^{\vee} \triangleq \gvec(\T)$ %\triangleq (( g_{\omegavec}^{\top}, g_{\vvec}^{\top})^{\top})^{\land}
        , where $\gvec(\T)=(g_{\omegavec}^{\top}, g_{\vvec}^{\top})^{\top}\in\mathbb{R}^6$ is $\grad f(\T)$'s  vector representation, % is the Lie-algebra coordinate representation of the Riemannian gradient,  
        $\rvec_{ij}= \xvec_i - \T^{-1} \zvec_j\in \mathbb{R}^3$,  
    \begin{align}\label{eq:gradient}
    g_{\vvec}  \hspace{-2pt}&=\hspace{-2pt}{\color{white}+}-\sum_{ij}   c_{ij}k(\xvec_i, \T^{-1}\zvec_j) \boldsymbol{\cov}_{ij}^{-1} \rvec_{ij}:=\sum_{ij} g_{\vvec}^{(ij)},\hspace{-2pt}\\
     g_{\omegavec}  \hspace{-2pt}&=\hspace{-2pt} -\sum_{ij} c_{ij}k(\xvec_i, \T^{-1}\zvec_j)( \T^{-1} \zvec_j)^{\wedge}  \boldsymbol{\cov}_{ij}^{-1}\rvec_{ij}:=\sum_{ij} g_{\omegavec}^{(ij)}.\hspace{-2pt}\notag
     \end{align}
\end{lemma}
\begin{proof}
The proof is in the supplementary materials, and follows from defining $\boldsymbol{\gamma}(t) = \T\exp(t\xivec^{\wedge})$, evaluating $\mathrm{D}f(\boldsymbol{\gamma}(t) )[\xivec^\wedge]$, and using~\eqref{eq:equivalence} to identify the gradients.
\end{proof}

\begin{algorithm}[t!]
\caption{Generalized CVO (\textbf{\ourmethod}).}
\label{alg:IF}
\begin{algorithmic}[1]
\State \textbf{receive} $\T^{(0)}$,  $\mu_0$, $k_{\max}$, $\epsilon_{\mathrm{tol}}$, $X$, $Z$
%\State Estimate point-wise covariance for $X$ and $Z$ 
\While{$\|\xivec^{(k)}\|_1\geq \epsilon_{\mathrm{tol}}$ \textbf{and} $k < k_{\max} $}
\State Transform $Z$ with ${(\T^{(k)})}^{-1}$\label{alg:transform}
\State Compute $\{\cov_{ij}\}$ with $\R$ from $\T^{(k)}$\hfill\eqref{eq:covariance}
\State Evaluate $f$ at $\T^{(k)}$ using $\{\cov_{ij}\}_{ij}$\hfill\eqref{eq:rkhs_loss}
%\State Evaluate $\nabla_{\xivec} F$ at $\T^{(k)}$ \hfill\eqref{eq:vgrad}
\State Evaluate $\grad f(\T^{(k)})$ \hfill\eqref{eq:gradient}
\If{using 1st-order update}
\State Evaluate $\mu^{(k)}$ using a line-search \hfill\eqref{eq:linesearch}
\State  $ {\xivec}^{(k)} = \mu^{(k)} \grad f(\T^{(k)})^{\lor} $\hfill~\eqref{eq:fo_update}
\Else
\State Evaluate $\hess_{\mathrm{GN}}f(\T^{(k)})$ \hfill\eqref{eq:hessian_gn}
\State $ {\xivec}^{(k)} \hspace{-2pt}=\hspace{-2pt} -\hess_{\mathrm{GN}}f(\T^{(k)})^{-1}\grad f(\T^{(k)})^{\lor}$\hspace{-2pt}\hfill \eqref{eq:gn_update}
\EndIf
\State Update $\T^{(k+1)}\leftarrow  \T^{(k)}\exp( [{\xivec}^{(k)}]^{\land})$
\State $k\leftarrow k+1$
%\If{the objective~\eqref{eq:rkhs_loss} stabilizes}
%\If{$\|\T^{(k+1)}(\T^{(k)})^{-1} -  \boldsymbol{I}\|_F\leq \epsilon_{\mathrm{step}}$}
%\If{$F$ has converged}
%\State $\ell \leftarrow \alpha\ell$
%\EndIf
\EndWhile
\State \textbf{return} $\T^{(k)}$
\end{algorithmic}
\end{algorithm}

%\subsection{Momentum with Riemannian Gradients}
%Gradient descent with momentum has shown better convergence speed in several applications~\cite{qian1999momentum, campos2024momentum}. In the problem (\PCRH), the step is a Lie algebra $\xivec^{\wedge}\in \mathfrak{se}(3)$, and the gradient from two different iterations lie in two different tanget spaces, i.e. $\xivec^{(k-1)}\in T_{\T^{(k-1)}}\mathsf{SE}(3), \xivec^{(k)}\in T_{\T^{(k)}}\mathsf{SE}(3)$. To interpolate the two gradients, we need to transport the gradient from the previous iteration's tangent space into the current iteration's %with a Parrallel Transport~\cite{boumal2023optmfd}, via the Adjoint operator~\cite{campos2024momentum}
%\begin{equation}\label{eq:momentum}
%\bar{\xivec}^{(k)} = (1-\gamma)\xivec^{(k)} + \gamma \Ad_{\exp{ (\T^{(k-1)})^{-1} \T^{(k)}}} \xivec^{(k-1)},
%\end{equation} 
%where $\gamma>0$ is a hyperparameter. In Sec.~\ref{sec:exp:ablations}, we explore the influence of this parameter in a set of ablation studies.

\subsection{A first-order method}
With Lemma~\ref{lem:gradient}, we can implement first-order optimizers to solve~\eqref{eq:PCRH} with the anisotropic kernels in~\eqref{eq:kernel}. To perform an update in the ascent direction, we calculate the step size $\mu$ using a line search along the lines of~\cite{clarkmaani20}. Specifically, we use a fourth-order expansion of~\eqref{eq:rkhs_loss} in the direction $\xivec^{\wedge}$:
\begin{align}
  G(t)\triangleq f(\T\exp(t\xivec^{\wedge})) \approx 
\sum_{ij} 
c_{ij} k_{ij} 
\sum_{k=0}^4 g^{k}_{ij} t^k,
\end{align} 
with the factors $g_{ij}^k$ in the supplementary materials. We find the maximum value of this polynomial near 0 by computing the roots of  $G(t)'$ in $t$.
\begin{subequations}\label{eq:linesearch}The smallest positive real root is chosen. If none are found, a minimum step $\mu_0>0$ is used
\begin{align}
t^{\star}&=\min\{t\in\mathbb{C}\;|\;G^{\prime}(t)=0, \mathrm{imag}(t) = 0\},\\
\mu &= \max\{\mu_0, t^{\star}\}.
\end{align}
\end{subequations}
The Riemannian gradient ascent step is implemented as% as an iteration
\begin{align}
\label{eq:fo_update}
\T^{(k+1)} = \T^{(k)}\exp(\mu^{(k)}\grad f(\T^{(k)})),
\end{align}
where $\mu^{(k)}$ is computed by~\eqref{eq:linesearch} at $\T^{(k)}$. % We refer to this method as \ourmethod-1, and give an algorithm sketch in Alg.~\ref{alg:IF}.

\subsection{A second-order method}
The first-order solver can be slow, as shown in Sec.~\ref{sec:exp:ablations}, when the same step size is applied to all dimensions of $\T$ without using curvature information.  As~\eqref{eq:rkhs_loss} is observed locally quadratic on $\manifold$ (see Sec.~\ref{sec:exp:lossslices}), we propose a Riemannian Gauss-Newton (GN) method using an approximated Riemannian Hessian. In each iteration, we first transform $Z$ with the current estimate as $\T^{-1}Z$. Then we denote 
\begin{align}
\label{eq:f_ij_J_j}
f_{ij} &\triangleq c_{ij}k(\xvec_i, \T^{-1}\zvec_j)
, \quad J_{j}\triangleq \begin{bmatrix}
-(\T^{-1} z_j)^{\wedge} & \idd
\end{bmatrix}\in\mathbb{R}^{3\times 6},\nonumber
\end{align}
%and note that 
\iffalse
\begin{align}
    \mathrm{D}(\rvec_{ij})[\xivec^{\wedge}]=\xivec^{\wedge} \T^{-1}\zvec_j,\quad \mathrm{D}(f_{ij})[\boldsymbol{\xi}^{\wedge}]=\langle \begin{bmatrix}
        g_{\omegavec}^{(ij)}\\
        g_{\vvec}^{(ij)}
\end{bmatrix}^{\wedge},\boldsymbol{\xi}^{\wedge}\rangle .
\end{align}
\fi
\vspace{-20pt}
\begin{align}
    \mathrm{D}(\rvec_{ij})[\xivec]&=J_j\xi %\xivec^{\wedge} \T^{-1}\zvec_j,%\\ %\mathrm{D}(f_{ij})[\boldsymbol{\xi}^{\wedge}]&=\langle ({g_{\omegavec}^{(ij)}}^{\top},{g_{\vvec}^{(ij)}}^{\top})^{\top})^{\wedge},\boldsymbol{\xi}^{\wedge}\rangle
\end{align}
%\langle J_j, \xivec^{\wedge}\rangle.  
%with the gradients defined in~\eqref{eq:gradient}. 
\iffalse
Following the vector-form Hessian  in~\eqref{eq:hess_se3_vec}, by the product rule,
\begin{subequations}
\begin{align}
    \mathrm{D}(&\gvec f(\T))[ \xivec^{\wedge}]=\sum_{ij}  J_{j}^{\top}\cov_{ij}^{-1}\rvec_{ij}\mathrm{D}(f_{ij})[ \xivec^{\wedge}] + \\
    &   f_{ij}\mathrm{D}(J_j^{\top})[\xivec^{\wedge}] \cov_{ij}^{-1}\rvec_{ij}  +  f_{ij}J_j^{\top} \mathrm{D}(\cov_{ij}^{-1})[\xivec^{\wedge}]\rvec_{ij}\\&+f_{ij}J_j^{\top}\cov_{ij}^{-1}\mathrm{D}(\rvec_{ij})[\xivec^{\wedge}].   \label{eq:direction_derivative_gradf}
\end{align}
\end{subequations}
\fi
By~\eqref{eq:hess_se3_vec}, it suffices to derive $\mathrm{D}(\gvec (\T))[ \xivec]$ in left-trivialized coordinates and then append the connection term from~\eqref{eq:koszul}.
\begin{subequations}
\begin{align}
\mathrm{D}(&\gvec (\T))[ \xivec]=-(\sum_{ij}  J_{j}^{\top}\cov_{ij}^{-1}\rvec_{ij}\mathrm{D}(f_{ij})[ \xivec] + \\
    &   f_{ij}\mathrm{D}(J_j^{\top})[\xivec] \cov_{ij}^{-1}\rvec_{ij}  +  f_{ij}J_j^{\top} \mathrm{D}(\cov_{ij}^{-1})[\xivec]\rvec_{ij}\\&+f_{ij}J_j^{\top}\cov_{ij}^{-1}\mathrm{D}(\rvec_{ij})[\xivec]).   
\label{eq:direction_derivative_gradf}
\end{align}
\end{subequations}
We approximate this directional derivative by the dominant last term in~\eqref{eq:direction_derivative_gradf}, as discussed in Remark~\ref{rem:approx}, resulting in
\begin{subequations}\label{eq:hessian_gn}
\begin{align}
\hspace{-6pt}(\hess& f(\T)[ \xivec^{\wedge}])^{\vee}\approx
(\hess_{\mathrm{GN}}f(\T)[ \xivec^{\wedge}])^{\vee}\\
&:=\hspace{-2pt}-\sum_{ij}
%H_{\mathrm{GN}}f(\T)_{ij}[ \xivec^{\wedge}]
f_{ij}J_{j}^{\top}\cov_{ij}^{-1}   J_{j} \xivec 
+ 
\Gamma^{\vee}(\xivec^{\wedge}\hspace{-2pt}, \grad f(\T)),\hspace{-4pt}
%g_{\mathrm{GN}}&=\Ad_{h}\sum_{ij}J^{\top}_{ij}\cov_{ij}^{-1}r_{ij}
\end{align}
\end{subequations}
where, by Lemma~\ref{lem:gradient} and~\eqref{eq:gamma_se3}, the  correction term is
\begin{align}
  \Gamma^{\vee}(  \xivec^{\wedge}, \grad f(\T))
  %=& 
  %\dfrac{1}{2}\Big([\xivec^{\wedge}, \grad f(\T)]-\ad_{\xivec}^{*}\grad f(\T)\\ 
  %\nonumber
  %&\quad\quad\quad-\ad_{\grad f(\T)}^*\xivec \Big)\\
  %\nonumber
  %=&
  %\dfrac{1}{2}\Big(-\ad_{\grad f(\T)} \xivec +\begin{bmatrix}
  %     0 & 0\\
  %    -g_v^{\wedge} &  g_{\omega}^{\wedge}
  %\end{bmatrix} \xivec \Big)\\
  =&\begin{bmatrix}
       -\frac{1}{2}g_{\omegavec}^{\wedge} & 0\\
      -g_{\vvec}^{\wedge} & 0
  \end{bmatrix}\xivec.\label{eq:connection_se3}
\end{align}
%Its derivation is in the supplemental material.
%Near local minima with small motions, we can also approximate $\Gamma(\xivec, \grad f(\T))\approx 0$ (\cite{boumal2023optmfd}, Proposition 5.3). 
As we are maximizing the object, 
the step update of the second-order solver is implemented as
\begin{align}
\label{eq:gn_update}
\hspace{-2pt}\T^{(k+1)} \hspace{-2pt}=\hspace{-2pt}\T^{(k)}\hspace{-2pt}\exp([-(\hess_{\mathrm{GN}}f(\T^{(k)}))^{-1} \grad f(\T^{(k)})^{\lor}]^{\land}).
\end{align}
We refer to~\eqref{eq:gn_update} as \ourmethod-2, as summarized in Alg.~\ref{alg:IF}.

\begin{remark}\label{rem:approx}
In the implementation of \ourmethod-2, the covariances $\boldsymbol{\Sigma}_{ij}$ are updated once per iteration and treated as constants within each step update. We further retain only the last term in~\eqref{eq:direction_derivative_gradf}. This inexact Hessian approximation is motivated by inexact GN schemes~\cite{gratton2007approximate} and supported by the KITTI ablation results in Sec.~\ref{sec:exp:ablations}: using the exact Hessian in~\eqref{eq:gn_update} provides marginal gains in convergence speed while increasing per-iteration cost, resulting in higher overall compute time for solving~\eqref{eq:PCRH}.
\end{remark}

\begin{remark}\label{rem:CVO}
By instead letting $\cov(\xvec,\zvec)\equiv \idd$ for all $\xvec,\zvec\in\mathbb{R}^3$, we obtain the kernel formulations of CVO~\cite{clarkmaani20} and SemanticCVO~\cite{zhang2021new}. As such, the exact and approximate GN method of \ourmethod can be applied in these settings as well.
\end{remark}
\begin{remark}
%$\grad f$ is generally not a left-invariant vector field for the function $f$ defined on $\mathsf{SE(3)}$. However,  in  Algorithm~\ref{alg:IF} (line~\ref{alg:transform}), we transform the point cloud $Z$ with the last iteration's estimate $(\T^{(k)})^{-1}$, so the new iteration's registration  takes place between $X$ and $(\T^{(k)})^{-1} Z$, from identity. We can thus avoid translating the vector fields even if the registration result is far from the identity.
In general, $\grad f$ is not a left-invariant vector field for a function $f$ defined on $\manifold$. However, in Alg.~\ref{alg:IF}, we re-express the problem at each iteration by transforming $Z$ with the inverse of the previous iterate, $(\T^{(k)})^{-1}$. Consequently, each registration step is performed from the identity, aligning with $(\T^{(k)})^{-1}Z$. This approach avoids transporting the gradients from $T_{\T^(k)}\mathcal{G}$ back to $\liealgebra$. %on $\manifold$ using the adjoint operators in~\eqref{eq:adjoints}.
\end{remark}

\section{Experiments}

\newcommand\bestcolor{black!20}
\newcommand\secondbestcolor{black!10}
\newcommand\cellcoloring[3]{
\edef\tempcoord{#2-#3}%
\rectanglecolor{#1}{\tempcoord}{\tempcoord}}

\begin{table*}[t]
\caption{Registration errors (translation and rotation) for different methods on standard KITTI sequences.}
\vspace{-10pt}

\centering
\resizebox{\textwidth}{!}{\begin{NiceTabular}{c|cc|cc|cc|cc|cc|cc|cc|cc}[colortbl-like,
code-before =% 
{
    \cellcoloring{\bestcolor}{14}{2}%
    \cellcoloring{\bestcolor}{15}{2}%
    \cellcoloring{\secondbestcolor}{14}{4}%
    \cellcoloring{\secondbestcolor}{15}{4}%
    \cellcoloring{\secondbestcolor}{14}{3}%
    \cellcoloring{\secondbestcolor}{15}{3}%
    \cellcoloring{\bestcolor}{14}{7}%
    \cellcoloring{\bestcolor}{15}{7}%
}
]
\toprule
& \multicolumn{2}{c}{\ourmethod-2 (\textbf{Ours})} 
& \multicolumn{2}{c}{\ourmethod-1 (\textbf{Ours})} 
& \multicolumn{2}{c}{Fast-VGICP~\cite{koide2021voxelized}} 
& \multicolumn{2}{c}{NDT~\cite{magnusson2009three}} 
& \multicolumn{2}{c}{GICP~\cite{segal2009gicp}} 
& \multicolumn{2}{c}{ICP~\cite{besl1992method}} 
& \multicolumn{2}{c}{CVO~\cite{MGhaffari-RSS-19, zhang2021new} (1st-order)} & \multicolumn{2}{c}{GeoTransformer~\cite{qin2023geotransformer}}\\
Sequence 
& Trans. & Rot. 
& Trans. & Rot. 
& Trans. & Rot. 
& Trans. & Rot. 
& Trans. & Rot. 
& Trans. & Rot. 
& Trans. & Rot. 
& Trans. & Rot.\\
\midrule
% Code-before to find and highlight min in each column
\texttt{00} &1.0649 &	0.0058 & 1.0060 & 0.0057 & 1.1653 & 0.0055 & 57.6272 & 0.2386 & 28.2713 & 0.0249 & 6.7379 & 0.0291 &1.8087	&0.0128 & - & -\\
\texttt{01} & 2.0752&	0.0066 & 1.9665 & 0.0065 & 8.5028 & 0.0080 & 94.7583 & 0.0928 & 94.1168 & 0.0825 & 21.3254 & 0.0158& 4.3708 &	0.0148 & - & -\\
\texttt{02} &1.4174 &	0.0062 &  1.4452 &	0.0072 & 1.6059 & 0.0055 & 68.3716 & 0.2149 & 34.4085 & 0.0483 & 9.5617 & 0.0307 & 1.6152&	0.0087 & - & -\\
\texttt{03} & 1.5743 &	0.0087 &2.7747 & 0.0179 & 1.1627 & 0.0073 & 87.5822 & 0.1238 & 1.7717 & 0.0058 & 10.3611 & 0.0449 & 2.1206&	0.0120 & - & -\\
\texttt{04} & 0.4718&	0.0070 & 0.4296 & 0.0050 & 1.3032 & 0.0078 & 97.1345 & 0.0083 & 100.3140 & 0.0034 & 9.3629 & 0.0404 & 1.2642&	0.0152 & - & -\\
\texttt{05} & 1.0557&	0.0060 &2.1047 & 0.0116 & 1.0113 & 0.0053 & 54.8189 & 0.2458 & 27.6153 & 0.0092 & 3.4138 & 0.0160 & 1.1743&	0.0075 & - & -\\
\texttt{06} & 1.1641&	0.0070 &0.7694 & 0.0066 & 0.8082 & 0.0049 & 59.0391 & 0.2848 & 29.2058 & 0.0036 & 4.7885 & 0.0192 &0.8103&	0.0079 & - & -\\
\texttt{07} & 1.1106 &	0.0065 & 0.5777 & 0.0045 & 0.7967 & 0.0052 & 57.6573 & 0.3708 & 33.9593 & 0.0051 & 4.0295 & 0.0262 & 1.0623&	0.0095 & - & -\\
\texttt{08} & 1.3356&	0.0055& 1.5383 & 0.0068 & 1.2516 & 0.0053 & 57.6909 & 0.2112 & 24.6989 & 0.0051 & 9.6332 & 0.0448 & 1.9684&	0.0103 & 6.8211  & 0.0189\\
\texttt{09} & 2.3890&	0.0091& 0.9627 & 0.0047 & 1.1934 & 0.0052 & 71.9249 & 0.2264 & 25.7434 & 0.0251 & 8.5134 & 0.0308 &3.6727&	0.0164 & 8.7523 & 0.0216\\
\texttt{10} & 1.6250&	0.0076& 1.7445 & 0.0089 & 1.3606 & 0.0055 & 78.7093 & 0.2092 & 12.4857 & 0.0070 & 10.2373 & 0.0377 & 1.8504&	0.0128 & 10.8313 & 0.0320\\
\midrule
Mean   & 1.3894 &	0.0069  & 1.3927 & 0.0078 & 1.8329 & 0.0060 & 71.3922 & 0.2024 & 37.5082 & 0.0200 & 8.9059 & 0.0305 &1.9744	&0.0116  & - & -\\
Std. & 0.5244 &	0.0012 &0.7217 & 0.0039 & 2.2245 & 0.0011 & 15.8960 & 0.0975 & 31.0308 & 0.0250 & 4.8482 & 0.0107 & 1.1024	&0.0031 & - & -\\
\bottomrule
\end{NiceTabular}}
\label{tab:kitti_results}
\end{table*}

We implement \ourmethod in CUDA and evaluate all methods on a computer equipped with an RTX 4090 GPU and  Intel Xeon CPU at 2.70GHz. We evaluate the proposed \ourmethod method on several public and private datasets, designed to benchmark the method on different sensor modalities with varying degrees of feature-sparsity and outliers. We primarily consider frame-to-frame tracking in consecutive frame pairs rather than a full visual odometry pipeline, %, and evaluate the average rotational and translational drift per $100$m, $200$m, $400$m, $800$m, as defined in the KITTI official development tool~\cite{Geiger2012kitti}. 
%First, we give intuition on the \ourmethod loss function in Sec~\ref{sec:exp:lossslices}. 
and evaluate \ourmethod in feature-rich urban driving scenarios in Sec.~\ref{sec:exp:featurerich} and more challenging feature-sparse vehicle racing environments in Sec.~\ref{sec:exp:featuresparse}. In addition, we evaluate tracking performance in an indoor RGB-D dataset in Sec.~\ref{sec:eth3d}. We study convergence and compute times in Sec.~\ref{sec:exp:ablations}, respectively, and present object registration results in Sec.~\ref{sec:object}.

\begin{figure}[t!]
\label{fig:kitti_00}
\centering
\begin{tikzpicture}
    %\node[anchor=north west] at (0,0) {\includegraphics[trim=0cm 0cm 0cm 0cm, clip=true, angle=0, width=0.45\textwidth]{fig5a_2D_kitti.pdf}};
    %\iffalse
    \node[anchor=north west] at (0,0) {\includegraphics[trim=0cm 0cm 0cm 0cm, clip=true, angle=0, width=0.35\textwidth]{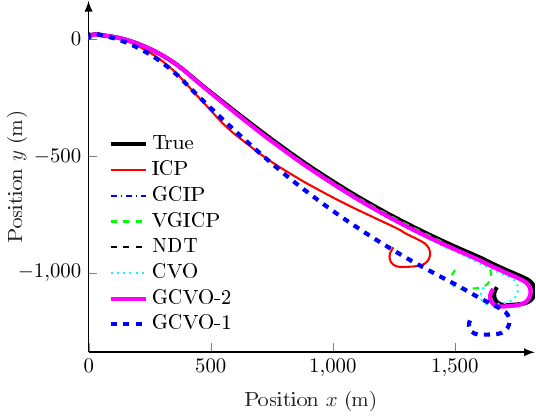}};
    \node[anchor=north west, draw=black, ultra thick, inner sep=0pt] at (4.8,0) {\includegraphics[trim=2.3cm 2cm 2cm 2cm, clip=true, angle=0, width=0.2\textwidth]{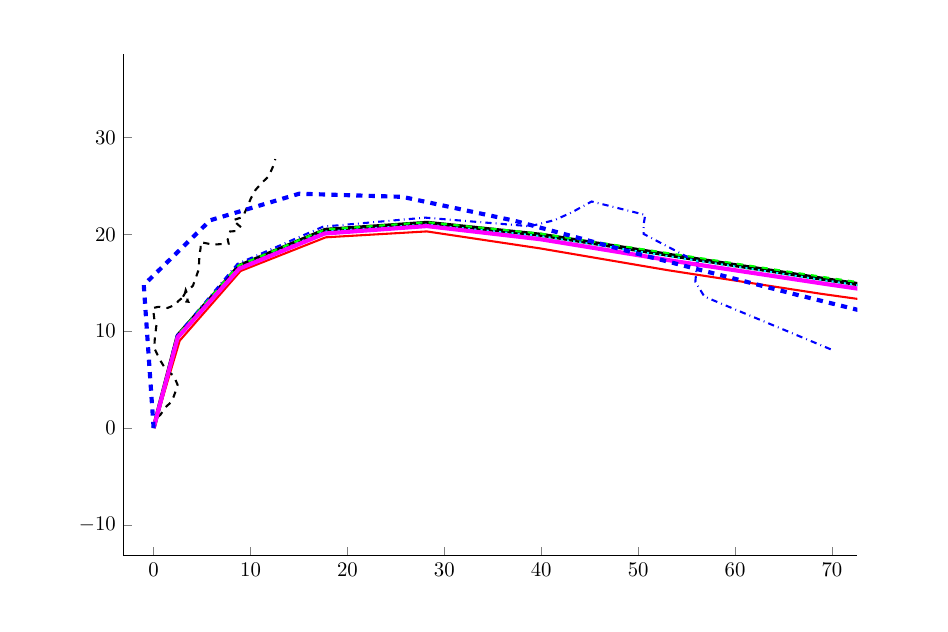}};
    \draw[black, thick] (1.15,-0.45) rectangle ++(0.3,-0.15);
    \draw[black, thick] (1.45,-0.525) -- (4.8,-1.1);
    %\fi
\end{tikzpicture}
\vspace*{-25pt}

\caption{\ourmethod versus baselines on KITTI sequence \texttt{01}.
}
\vspace*{-10pt} 
\label{fig:kitti_lidar_plots}
\end{figure}

\subsection{\lidar Tracking in Feature-rich Scenes}\label{sec:exp:featurerich}
First, we consider frame-to-frame registration for sequential frame pairs in the KITTI \lidar dataset~\cite{Geiger2012kitti}. This dataset contains $11$ \lidar sequences with ground truth poses in urban driving scenarios. Each frame is downsampled with a voxel filter of $0.25$m. %To evaluate registration performance   without  interference from the mapping frameworks, the trajectory is built with frame-to-frame registration on two consecutive frames. % The metrics used here are 
As evaluation metrics, we consider KITTI odometry benchmark's official Translational Error and  Rotational Error, computed over various subsequence lengths ($100$m, $200$m, $400$m, $800$m).   We compare \ourmethod to various ICP baselines, including ICP~\cite{besl1992method},  GICP~\cite{segal2009gicp} implemented with PCL~\cite{rusu20113d},  %more modern 
Fast-VGICP~\cite{koide2021voxelized}, and a GMM-based method NDT~\cite{magnusson2009three}. We also include the classical CVO~\cite{clarkmaani20, zhang2021new} as a correspondence-free baseline. We include GeoTransformer~\cite{qin2023geotransformer}, a learning-based method with strong global registration performance, but only report its test sequences \texttt{08-10} (training sequences are omitted).

Quantitative results are summarized in Table~\ref{tab:kitti_results}, with qualitative results on sequence \texttt{01} % and \texttt{01} 
 in~Fig.~\ref{fig:kitti_lidar_plots}. \ourmethod reduces translation errors and is on par with Fast-VGICP with respect to  rotational errors. %In {\color{red}8 out of the 11 sequences}, \ourmethod has the best translation drift errors.
 In particular, for the less feature-dense highway sequences \texttt{01} and \texttt{04} where correspondences are difficult to find, the proposed method outperforms the baselines by a substantial margin.

\def\sep{0.1}
\def\wid{2.9}
\def\hgtA{2}
\def\hgtB{3}

\begin{figure}[t!]
    \includegraphics[width=0.4\textwidth]{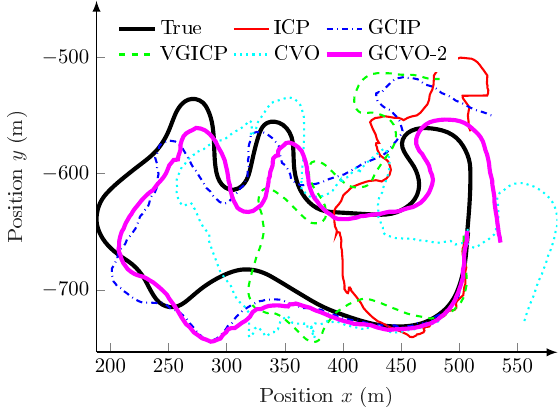}
    \includegraphics[width=0.4\textwidth]{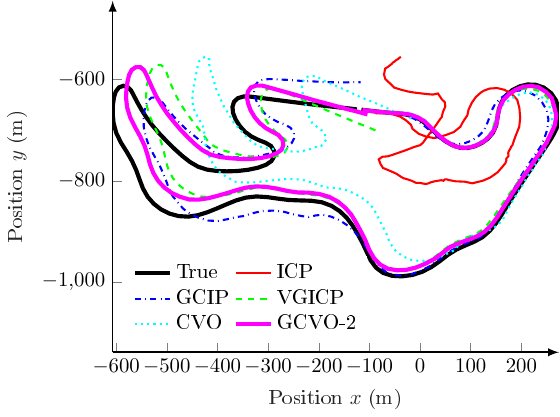}
    \vspace{-10pt}
    
    \caption{Ground truth and frame-to-frame tracking with \ourmethod-2 on the racing dataset (\texttt{Dirt Track} top, \texttt{Race track} bot.).}
    \vspace*{-10pt} 
    % (top) and \texttt{Race Track} (bottom).}
    \label{fig:qualitativeRaceResults}
\end{figure}

\iffalse
\begin{figure*}[t!]
    \centering
    \resizebox{\textwidth}{!}{\begin{tikzpicture}[every text node part/.style={align=center}]
    \foreach \i/\label in {
    0/dirt_track.png, 
    1/dirt_track_illustraiton.png,
    2/race_track_illustration.png,
    3/th_2mile.png
    } {
        \node[anchor=south west] at ({\i*\wid+\i*\sep},{0}) {\includegraphics[width=2.9cm, height=2cm]{\label}};
    }
    \foreach \i/\label in {
    0/fig4a_2D_thunderhill.pdf,
    1/fig4b_2D_thunderhill.pdf} {
        \node[anchor=north west] at ({2*\i*\wid+2*\i*\sep},{0.15cm}) {\includegraphics[width=5.9cm]{\label}};
    }
    %\foreach \i/\label in { 0/{Dirt Track}, 1/{Race Track}} {
    %    \node[] at (2*\i*\wid+2*\i*\sep+\wid+0.5*\sep, \hgtA+0.3) {\label};
    %}
    \end{tikzpicture}}
    \vspace{-20pt}
    
    \caption{Performance in the feature-sparse experiments. Camera and \lidar (top) and the OxTS trajectory (black) with frame-to-frame trajectories with the baselines and \ourmethod (\textcolor{magenta}{magenta}) for the  \texttt{Dirt Track} (left) and \texttt{Race Track} (right).}
    \label{fig:qualitativeRaceResults}
\end{figure*}
\fi

\newcommand\tmpnum{{\color{red}XX.XX}}
\begin{table*}[t]
\centering
\caption{Registration errors (translation and rotation) for different methods on the three feature-sparse racing sequences.}
\vspace{-10pt}

\resizebox{0.9\textwidth}{!}{\begin{NiceTabular}{c|cc|cc|cc|cc|cc|cc}[colortbl-like,
code-before =%
{
    % trans errors
    \cellcoloring{\bestcolor}{3}{6}%
    \cellcoloring{\bestcolor}{4}{2}%
    \cellcoloring{\bestcolor}{5}{2}%
    \cellcoloring{\bestcolor}{6}{2}%
    \cellcoloring{\bestcolor}{7}{2}%
    \cellcoloring{\secondbestcolor}{3}{2}%
    \cellcoloring{\secondbestcolor}{4}{8}%
    \cellcoloring{\secondbestcolor}{5}{6}%
    \cellcoloring{\secondbestcolor}{6}{12}%
    \cellcoloring{\secondbestcolor}{7}{12}%
     % rot errors
    \cellcoloring{\bestcolor}{3}{7}%
    \cellcoloring{\bestcolor}{4}{3}%
    \cellcoloring{\bestcolor}{5}{3}%
    \cellcoloring{\bestcolor}{6}{3}%
    \cellcoloring{\bestcolor}{7}{3}%
    \cellcoloring{\secondbestcolor}{3}{3}%
    \cellcoloring{\secondbestcolor}{4}{13}%
    \cellcoloring{\secondbestcolor}{5}{7}%
    \cellcoloring{\secondbestcolor}{6}{13}%
    \cellcoloring{\secondbestcolor}{7}{13}%
}
]
\toprule
& \multicolumn{2}{c}{\ourmethod-2 (\textbf{Ours})} 
& \multicolumn{2}{c}{\ourmethod-1 (\textbf{Ours})} 
& \multicolumn{2}{c}{Fast-VGICP~\cite{koide2021voxelized}}  
& \multicolumn{2}{c}{GICP~\cite{segal2009gicp}} 
& \multicolumn{2}{c}{ICP~\cite{besl1992method}} 
& \multicolumn{2}{c}{CVO (2nd-order)}\\
Sequence 
& Trans. & Rot. 
& Trans. & Rot. 
& Trans. & Rot. 
& Trans. & Rot. 
& Trans. & Rot. 
& Trans. & Rot. \\
\midrule
% Code-before to find and highlight min in each column
\texttt{Skid Pad} &2.5076 &	0.0149	& 7.1536 & 0.0404 & 1.8999&	0.0138&		29.9657&	0.4065&		30.4623&	0.0917&		4.2126&	0.0270 \\
\texttt{Dirt Track} & 6.0791	&0.0167	& 20.8969 & 0.0772 &	26.3557	&0.0807&	9.7605&	0.0557	&	33.1731	&0.1404&		15.0181&	0.0492 \\
\texttt{Race Track} & 4.3702&	0.0101& 31.5965 & 0.0479 &		6.3208&	0.0133&		8.4150&	0.0193	&	40.1477&	0.0667	&	10.4601&	0.0183\\
\midrule
Mean     & 4.3189	&0.0139& 19.8823&	0.0552 &		11.5255&	0.0359&	16.0471&	0.1605&	34.5944&	0.0996&	9.8969&	0.0315 \\
Std. & 1.7863	&0.0034& 12.2530 &	0.0195	&13.0322&	0.0388&	12.0727	&0.2138	&	4.9967&	0.0375&		5.4247&	0.0159\\
\bottomrule
\end{NiceTabular}}
\label{tab:race_results}
\end{table*}

\subsection{\lidar Tracking in Feature-sparse Scenes}\label{sec:exp:featuresparse}
To evaluate the robustness of \ourmethod in high-speed, feature-sparse environments, we collect a dataset with a Velodyne 128-beam \lidar at a closed-course race track. %\footnote{Thunderhill Raceway is located at 5250 CA-162, Willows, CA, USA.}
The dataset consists of three sequences with ground truth poses from the RTK GPS system~\cite{oxts2024}, including one sequence on a large empty skid pad (\texttt{Skid Pad}), one sequence on a 2-mile paved racetrack (\texttt{Race Track}), and one sequence on a dirt track (\texttt{Dirt Track}). Images from an on-board camera and sample point clouds are provided in the supplementary materials. To achieve convergence times below $100$ms, we downsample each frame using a voxel size of $2$m. As an additional baseline, we run the classical CVO with isotropic kernels using the proposed second-order solver (see Remark~\ref{rem:CVO}). All the solvers are limited to a maximum of $200$ iterations, mirroring typical \lidar odometry settings~\cite{xu22fastlio2}. We use the same standard KITTI metrics as Sec.~\ref{sec:exp:featurerich}. Quantitative results are reported in Table~\ref{tab:race_results}, using a subset of the baselines in Sec.~\ref{sec:exp:featurerich}.

In the \texttt{Skid Pad} sequence, few features fall within the \lidar’s effective range, leaving the registration vulnerable to incorrect matches~\cite{besl1992method,segal2009gicp}. The \texttt{Race Track} represents a different challenge due to high-speed driving, leading to large movements between consecutive frames. Furthermore, the sequence covers long distances (2 miles), making frame-to-frame registrations prone to long-term drifts. %Despite this, \ourmethod-2 results in good registrations in the frame-to-frame setting. 
Finally, the \texttt{Dirt Track} poses a unique challenge due to the high vibration of the chassis and significant dust clouds that emerge behind the car. Despite these challenges, \ourmethod-2 demonstrates improved robustness in these scenes compared to baselines. Compared with surface-informed techniques (G-CVO, GICP) general perform better than methods using point-wise residuals like ICP and classical CVO perform, even when CVO is implemented with a second-order solver.

\begin{table*}[t]
\centering
\caption{Tracking errors on the five ETH3D RGB-D sequences: RPE RMSE ($\Delta{=}1$) and APE RMSE for each method.}
\vspace*{-10pt}

\resizebox{\textwidth}{!}{%
\begin{NiceTabular}{c|cc|cc|cc|cc|cc|cc}[colortbl-like,
code-before =%
{
    % --- row 3 (sfm_bench): best/second-best per metric ---
    \cellcoloring{\bestcolor}{3}{8}\cellcoloring{\secondbestcolor}{3}{2} % RPE
    \cellcoloring{\bestcolor}{3}{3}\cellcoloring{\secondbestcolor}{3}{9} % APE
    % --- row 4 (plant_1) ---
    \cellcoloring{\bestcolor}{4}{2}\cellcoloring{\secondbestcolor}{4}{4} % RPE
    \cellcoloring{\bestcolor}{4}{3}\cellcoloring{\secondbestcolor}{4}{5} % APE
    % --- row 5 (table_3) ---
    \cellcoloring{\bestcolor}{5}{2}\cellcoloring{\secondbestcolor}{5}{4} % RPE
    \cellcoloring{\bestcolor}{5}{3}\cellcoloring{\secondbestcolor}{5}{5} % APE
    % --- row 6 (sfm_lab_room_1) ---
    \cellcoloring{\bestcolor}{6}{8}\cellcoloring{\secondbestcolor}{6}{2} % RPE
    \cellcoloring{\bestcolor}{6}{9}\cellcoloring{\secondbestcolor}{6}{7}  % APE
    % --- row 7 (planar_2) ---
    \cellcoloring{\bestcolor}{7}{2}\cellcoloring{\secondbestcolor}{7}{4} % RPE
    \cellcoloring{\bestcolor}{7}{5}\cellcoloring{\secondbestcolor}{7}{3}  % APE
    % --- row 8 (Mean) ---
    \cellcoloring{\bestcolor}{8}{8}\cellcoloring{\secondbestcolor}{8}{2} % RPE
    \cellcoloring{\bestcolor}{8}{3}\cellcoloring{\secondbestcolor}{8}{9} % APE
    % --- row 9 (Variance) ---
    \cellcoloring{\bestcolor}{9}{8}\cellcoloring{\secondbestcolor}{9}{2} % RPE
    \cellcoloring{\bestcolor}{9}{3}\cellcoloring{\secondbestcolor}{9}{9} % APE
}
]
\toprule
& \multicolumn{2}{c}{\ourmethod-2 (\textbf{Ours})}
& \multicolumn{2}{c}{\ourmethod-1 (\textbf{Ours})} 
& \multicolumn{2}{c}{Fast-VGICP~\cite{koide2021voxelized}}  
& \multicolumn{2}{c}{GICP~\cite{segal2009gicp}} 
& \multicolumn{2}{c}{ICP~\cite{besl1992method}} 
& \multicolumn{2}{c}{GeoTransformer~\cite{qin2023geotransformer}}\\
Sequence 
& RPE RMSE  & APE RMSE 
& RPE RMSE & APE RMSE 
& RPE RMSE & APE RMSE 
& RPE RMSE & APE RMSE 
& RPE RMSE & APE RMSE & RPE RMSE & APE RMSE\\
\midrule
\texttt{sfm\_bench}        & 0.0157 & 0.3426 & 0.1730 & 7.1005 & 0.0178 & 1.0861 & 0.0058 & 0.3568 & 0.0259 & 1.2655 &  0.0217 & 1.2327\\
\texttt{plant\_1}          & 0.0021 & 0.0095 & 0.0022 & 0.0162 & 0.0217 & 0.1842 & 0.0034 & 0.0252 & 0.0086 & 0.1279 & 0.3572 & 2.2901 \\
\texttt{table\_3}         & 0.0046 & 0.1722 & 0.0054 & 0.2293 & 0.0429 & 0.5669 & 0.0145 & 0.6896 & 0.0269 & 1.0142  & - & - \\
\texttt{sfm\_lab\_room\_1} & 0.0187 & 0.3712  & 0.0332 & 0.3315 & 0.0189 & 0.2522 & 0.0041 & 0.0304 & 0.0232 & 0.3753 & 0.0210 & 0.3418\\
\texttt{planar\_2} &	0.0021 &	0.0610 & 0.0021 &	0.0540 &	0.0177 &	0.3883 &	0.0044 &	0.1429 &	0.0103 &	0.4007  & 0.0099 & 0.3638  \\
\midrule
Mean                       &	0.0086	&0.1913 & 0.0432 &	1.5463&	0.0238&	0.4956&	0.0064	&0.2490	&0.0190&	0.6367 & - & - \\
Std.                   &	0.0079&	0.1625 & 0.0737&	3.1075&	0.0108&	0.3611&	0.0046&	0.2806&	0.0088&	0.4798 & - & - \\
\bottomrule
\end{NiceTabular}}
\label{tab:eth3d_results}
\end{table*}

\begin{figure}[t!]
    \centering
    \includegraphics[width=0.8\columnwidth]{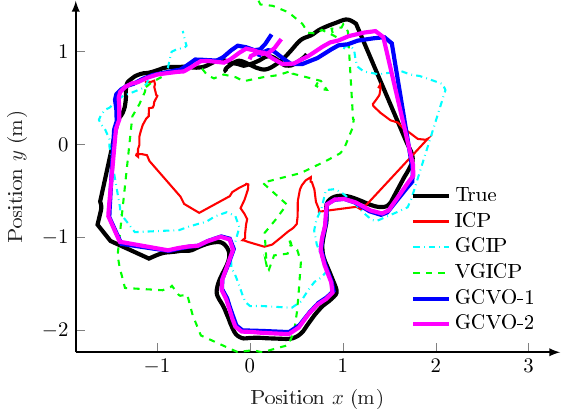}
    \vspace{-10pt}
    
    \caption{Ground truth (black) and frame-to-frame tracking using baselines and \ourmethod-1 and \ourmethod-2 on ETH3D (\texttt{table\_3}).}
    \label{fig:eth3d}
    \vspace{-10pt}
\end{figure}

\subsection{RGB-D Tracking in Indoor Scenes}
\label{sec:eth3d}
Next, we perform frame-to-frame registration on the indoor  dataset with RGB-D point clouds, ETH3D~\cite{schops2019badslam}, on the following five sequences: \texttt{sfm\_bench},  \texttt{plant\_1},  \texttt{table\_3},  \texttt{sfm\_lab\_room\_1}, \texttt{planar\_2}. We downsample each frame to around 3000 points. We assess tracking accuracy using the RMSE of Relative Pose Error (RPE) and Absolute Pose Error (APE) over the estimated trajectories with \texttt{evo}~\cite{grupp2017evo}. RPE (with 
$\Delta=1$) quantifies local drift in relative motion (position and orientation), while APE measures the global drift in absolute position against ground truth.

Quantitative results are shown in Table~\ref{tab:eth3d_results}, with qualitative results on \texttt{table\_3} in Fig.~\ref{fig:eth3d}. G-CVO has the best tracking accuracy in three sequences, and has the second-best performance for the other two. On average, G-CVO outperforms all baselines in terms of APE, and is the second-best method in terms of RPE.  % behind GICP. 

%\subsection{Large Registration with Learned  Descriptors}
%In loop closure and relocalization tasks, registration often lacks a reliable initialization and exhibits limited inter-scan overlap. We integrate \ourmethod with learned descriptors and evaluate large-motion registration on the 3DMatch~\cite{zeng20173dmatch} dataset. Specifically, we extract point-wise 32-bit  features from PREDATOR~\cite{huang2021predator} using its pretrained weights  and include them as feature label functions $\ell^{\mathbf{X}}$ in~\eqref{eq:point_cloud_function}~\cite{zhang2021new}. Instead of repeatedly refining pairwise correspondences, we select an initial set of inliers from PREDATOR and perform correspondence-free registration within these inliers. The baselines include RANSAC-1M~\cite{fischler1981ransac}, FGR~\cite{zhou2016fast}, TEASER++~\cite{yang2020teaser}, and SC2-PCR~\cite{Chen_2022_sc2pcr}. The metrics include the Registration Recall (RR), Rotation Error (RE) in \texttt{deg}, and Translation Error (TE) in $\texttt{cm}$. The True-Positive instances of RR include the registration results whose rotation errors are less than $15^{\circ}$ and translation errors are less than $30\mathrm{cm}$.  The RR and TE are calculated over successful registrations. 

\subsection{Computational Properties and Approximations}\label{sec:exp:ablations}

To provide intuition, we study the convergence of \ourmethod on the KITTI \lidar  dataset.  
 We compare the number of iterations and total time to converge in Fig.~\ref{fig:ablation_convergence} and Fig.~\ref{fig:ablation_time_no_init}. We compare \ourmethod-1, \ourmethod-2, and a version of \ourmethod that implements the exact Hessian, referred to as \ourmethod-E.
 
 First, we consider the frame pair of indices \texttt{000000} and \texttt{000001} from KITTI sequence \texttt{03}, which has a ground truth relative pose $\T_{\mathrm{gt}}$ with a translation of approximately $0.95$ (m). We report the change of the Frobenius error $\|\T^{-1}_{\mathrm{gt}}\T^{(k)}-\idd\|_F$ with respect to the iterates $k$ in Fig.~\ref{fig:ablation_convergence}. \ourmethod-2 converges in only $\sim$30 iterations, and \ourmethod-1 requires more than 1200 iterations to reach the same tolerances. 
 The increase of the inner product objective $f(\T)$ correlates well with the decrease in error.
\begin{figure}[t!]
\centering
\includegraphics[width=0.8\columnwidth]{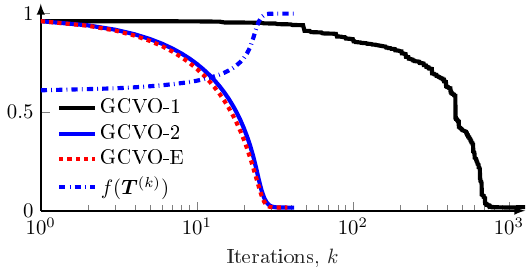}
\vspace{-10pt}

\caption{Normalized Frobenius error with the 1st/2nd-order optimizers in GCVO (black/\textcolor{black}{blue}/red) over the iterates on an input pair of KITTI sequence \texttt{03}. For reference, we report the inner-product objective of in GCVO-2 (\textcolor{black}{blue, dashed}).}
\label{fig:ablation_convergence}
\end{figure}

%This results in a significant difference in computation times 
In Fig.~\ref{fig:ablation_time_no_init}, we compare the convergence time, excluding IO preprocessing, averaged over a sequence of 100 frames from the KITTI sequence \texttt{04}, as a function of the number of points in the registration ($N_X$ and $N_Z$). The methods are evaluated both without an initial guess ($\T^{(0)}=\idd$) and with an initial guess (using a constant velocity assumption to initialize $\T^{(0)}$), as is typically done in visual and \lidar  odometry~\cite{xu22fastlio2}. %We refer to registrations initialized with the identity transformation as those without an initial guess, and those initialized using a constant-velocity assumption as with an initial guess.
GCVO-2 converges in less than $100$ms with around $10^4$ points %, which is on the order of the size of $|X|+|Z|$ 
in the frame-to-frame tracking of Sec.~\ref{sec:exp:featuresparse}. GCVO-1 is one order of magnitude slower in this setting. 
While \ourmethod-E and \ourmethod-2 converge to the same solution,   the approximation in \ourmethod-2 decreases compute time by $\sim$50\% compared to \ourmethod-E % across the problem instances, 
 irrespective of the quality of the initial guess. 

%\ray{add one more time comparison with init guess}

\begin{figure}
\centering
\includegraphics[trim=0cm 0cm 0cm 0cm, clip=true, width=0.75\columnwidth]{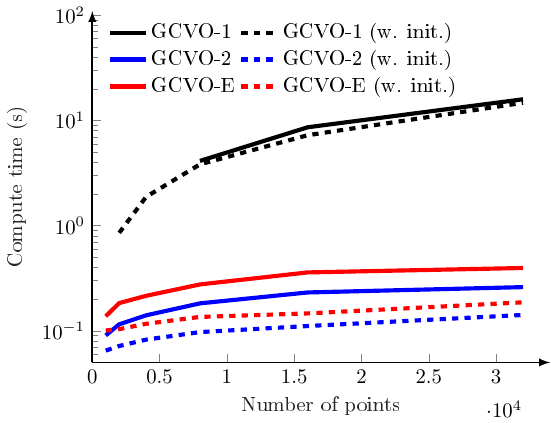}

  \caption{Average convergence time (excluding 
IO preprocessing) for \ourmethod-1/2/E on KITTI sequence \texttt{04}, with and without warm-starting the optimizer using a constant velocity assumption.}
\vspace{-20pt}
  \label{fig:ablation_time_no_init}
\end{figure}
The per-iteration complexity of G-CVO is $O(N_X N_Z)$, compared to $O(N_X \log N_Z)$ for KD-tree–based ICP. VGICP~\cite{koide2021voxelized} achieves additional speedups via voxel hashing, which is not used in G-CVO. To evaluate runtime scaling, we report time–accuracy trade-offs on seven KITTI sequences (\texttt{01-07}) with $N \in {4,8,16}$k voxel-downsampled points (Fig.~\ref{fig:speedup}). G-CVO and VGICP-GPU run on GPU, while others use CPU. Despite higher per-iteration cost, G-CVO attains higher accuracy than VGICP (Fig.~\ref{fig:speedup}) and converges faster than CVO’s first-order solver (Fig.~\ref{fig:ablation_convergence}) due to second-order optimization. Voxel downsampling takes $0.0413$ s per frame; combined with convergence time, G-CVO achieves 10 Hz tracking with 4k points.

%\vspace*{-13pt}

\begin{figure}[h!]
    \centering
   \includegraphics[width=0.5\columnwidth]{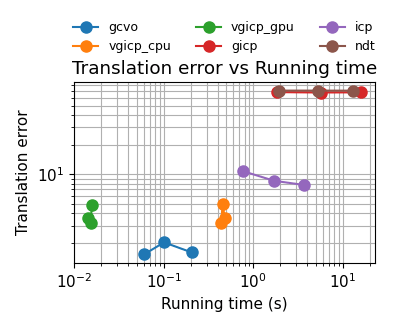}%
   \includegraphics[width=0.5\columnwidth]{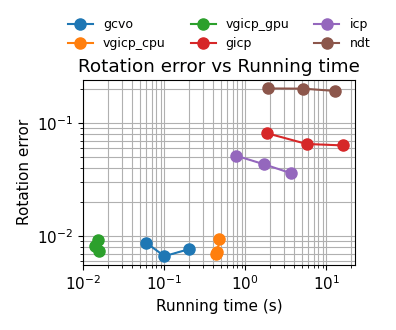}
   
  \vspace*{-15pt}
   
  \caption{\textit{Convergence time (IO/downsampling excluded) vs accuracy. Accuracy in KITTI metrics at $N=\{4,8,16\}\cdot 10^3$ points.}}
  
    \label{fig:speedup}
\end{figure}

%\vspace*{-12pt}

%%%%%%%%%%%%%%%%%%%%%%%%%%%%%%%%%%%%

\subsection{Object Registration}\label{sec:object}
%\noindent\rquestionbC{Q3}{
%The scope is limited to tracking.
%}
%7. 7qUH: Scope: limited to tracking. Other scenarios, such as object-centric registration or global registration with low-overlap
While the primary use of \ourmethod is in fine-grained local registration for tracking, we also evaluate its object-level registration capability on ModelNet40~\cite{wu20153d} (Tab.~\ref{tab:global}). Each model is randomly transformed with mean rotation and translation of $45^{\circ}$ and $0.5\mathrm{m}$, with maximum magnitudes of $80^{\circ}$ and $0.85\mathrm{m}$. We further apply zero-mean Gaussian noise ($\sigma^2=0.01$), uniformly-distributed outliers ($\gamma=0.2$), and randomly crop $30\%$ of the points using a plane in both point clouds. We compare against VGICP, RANSAC (50k iterations), and GeoTransformer~\cite{qin2023geotransformer}, a learning-based global registration method. G-CVO consistently outperforms baselines that are not learning-based. Without cropping or outliers, it surpasses GeoTransformer; under reduced overlap, performance degrades due to the absence of learned invariant features. Refining GeoTransformer outputs with G-CVO consistently improves pose accuracy.

\begin{table}[t!]
    \vspace{-10pt}
    \centering
    \setlength{\tabcolsep}{2pt}
    \caption{\textit{Mean object registration errors on  ModelNet40~\cite{wu20153d}}}
    \vspace{-10pt}
    
    \resizebox{\columnwidth}{!}{
    \begin{NiceTabular}{c|c|c|c|c|c|c}[colortbl-like,
code-before =%
{
    \cellcoloring{\bestcolor}{6}{2}
    \cellcoloring{\bestcolor}{6}{3}
    \cellcoloring{\bestcolor}{5}{4}
    \cellcoloring{\bestcolor}{5}{5}
    \cellcoloring{\bestcolor}{5}{6}
    \cellcoloring{\bestcolor}{5}{7}
    \cellcoloring{\secondbestcolor}{5}{2}
    \cellcoloring{\secondbestcolor}{5}{3}
    \cellcoloring{\secondbestcolor}{6}{4}
    \cellcoloring{\secondbestcolor}{6}{5}
    \cellcoloring{\secondbestcolor}{6}{6}
    \cellcoloring{\secondbestcolor}{6}{7}
    \cellcoloring{\secondbestcolor}{7}{4}
    \cellcoloring{\secondbestcolor}{7}{5}
    \cellcoloring{\secondbestcolor}{7}{6}
    \cellcoloring{\secondbestcolor}{7}{7}
    \cellcoloring{\bestcolor}{8}{4}
    \cellcoloring{\bestcolor}{8}{5}
    \cellcoloring{\bestcolor}{8}{6}
    \cellcoloring{\bestcolor}{8}{7}
}
]\toprule
    Method & \multicolumn{2}{c}{$\sigma^2=0.01$} & \multicolumn{2}{c}{$\sigma^2=0.01$, $\gamma=0.2$} &  \multicolumn{2}{c}{$\sigma^2=0.01$,  crop=$30\%$}\\
     (${\star}$: init with~\cite{qin2023geotransformer}) & rot. ($\mathrm{deg}$) & trans. ($\mathrm{m}$) & rot.($\mathrm{deg}$) & trans.($\mathrm{m}$)& rot.($\mathrm{deg}$) & trans.($\mathrm{m}$)\\\midrule
RANSAC-50k & 105.9665 & 0.5936 & 104.2570 & 0.5871 &111.6881 & 0.6616 \\
VGICP & 38.6575 & 0.4484 & 43.6178 &  0.4832 &41.6356  & 0.4643\\
GeoTransformer~\cite{qin2023geotransformer} & {0.6495} &  {0.0088} &   {0.9341} &	{0.0110}& {4.2401} & {0.0460} \\

GCVO-2{\color{white}$^{\star}$} (\textbf{Ours}) &    {0.2669} & {0.0024}  &  {3.0607} &  {0.0206}  &  {32.7753}  &  {0.2829} \\
\midrule
VGICP$^{\star}$ & - &  - & 1.0363 & 0.0103 &6.0881 & 0.1229\\
GCVO-2$^{\star}$ (\textbf{Ours})&   - & - & {0.3549} &	{0.0014}	   & {3.8465} & {0.0401}\\\bottomrule
    \end{NiceTabular}}
    
    \vspace{-12pt}
    
    \label{tab:global}
\end{table}

\section{Conclusions}

%{\color{red}Limitations. GPU is almost necessary, a limitation of the analysis (targeting odometry).}

%We propose a novel correspondence-free formulation of the \lidar registration problem. It effectively estimates anisotropic kernels to construct point cloud functions, thus providing precise alignment along the surface normals of \lidar scans. We demonstrate better accuracy and the speed of the proposed method versus the baselines, such as ICP-based methods and classical CVO, on public and self-collected \lidar datasets. 

We propose a correspondence-free, local-geometry-aware registration method to align point clouds by maximizing the inner product between their functional representations in an RKHS. In tracking, \ourmethod achieves lower drift than CVO and ICP-based baselines on public and self-collected \lidar/RGB-D datasets. It can also be used for local refinement in global registration tasks. %It achieves lower drifts than the classical CVO and ICP-based baselines on publicm self-collected \lidar and RGB-D odometry datasets. % It also shows the capability of aligning large  
The use of anisotropic kernels encourages surface-aware alignment, while the second-order Riemannian optimization enables fast convergence.

A limitation of the approach is relying on a GPU for fast evaluations of dense pairwise kernel functions and Jacobians. Additionally, the proposed solvers target odometry applications that assume sufficient overlap between the two inputs, which may not be respected in large motion or low-overlap registrations. %, which often requires deep learning-based frameworks.  
Future works include integrating the registration \ourmethod in complete visual and \lidar odometry systems as well as extending the objective function into differentiable layers of deep learning frameworks.

{
    \small
    \bibliographystyle{ieeenat_fullname}
    \bibliography{main}
}

%\iffalse
\clearpage

\newpage\appendix
\label{app:dataset}

%\section{Coefficients of the Line Search Expansion}

\section{Proof of Lemma~\ref{lem:gradient}}\label{app:gradient}
Let $\boldsymbol{\gamma}(t) = \T\exp(t\xivec^{\wedge}) $. We write the objective function $f$ in~\eqref{eq:rkhs_loss} as a function of this perturbation in $t$ as
\begin{align*}
f(\boldsymbol{\gamma}(t))=\sum_{ij} c_{ij}   k(\xvec_i,\exp(-t\xivec^{\wedge}) \T^{-1} \zvec_j),
\end{align*}
where the sum is over $i\in 1,...,|X|$ and $j=1,...,|Z|$, respectively. Denote $\rvec_{ij}:=\xvec_i - \T^{-1}\zvec_j$ and $f_{ij} := c_{ij}k(\xvec_i, \T^{-1}\zvec_j)$ as in~\eqref{eq:f_ij_J_j}. Taking the directional derivative, we obtain
\begin{align}\label{eq:lasteq}
    \mathrm{D}&f(\boldsymbol{\gamma}(t) )[\xivec^\wedge]  = \sum_{ij} c_{ij}  \mathrm{D}k(\xvec_i,\exp(-t\xivec^{\wedge}) \T^{-1} \zvec_j)[\xivec^\wedge] \notag\\
    =& -\dfrac{1}{2}\sum_{ij} c_{ij}  k(\xvec_i,  \T^{-1} \zvec_j)\lim_{t\rightarrow 0} \dfrac{1}{t}\{ \notag\\
    &\langle \xvec_i-  (\idd-t\xivec^\wedge+o(t^2))\T^{-1} \zvec_j,  \notag\\&
    {\boldsymbol{\Sigma }_{ij}^{-1}} (\xvec_i-  (\idd-t\xivec^\wedge+o(t^2)) \T^{-1} \zvec_j )\rangle \notag\\
    &-\langle \xvec_i-  \T^{-1}  \zvec_j, {\boldsymbol{\Sigma}_{ij}^{-1}} ( \xvec_i-  \T^{-1}  \zvec_j) \rangle\}  \notag\\
   % &=  \sum_{ij} c_{ij}  k(\xvec_i,  \T^{-1} \zvec_j)  \langle \xvec_i -  \T ^{-1} \zvec_j, \boldsymbol{\Sigma}(\xvec_i)^{-1} \xivec^{\wedge} \T^{-1} \zvec_j \rangle, \\
    %=&+\sum_{ij}c_{ij}k(\xvec_i,  \T^{-1} \zvec_j)\langle \xvec_i-  \T^{-1}  \zvec_j,  \boldsymbol{\Sigma}_{ij}^{-1} \xivec^{\wedge} \T^{-1} \zvec_j\rangle \notag \\
    =&-\sum_{ij}c_{ij}k(\xvec_i,  \T^{-1} \zvec_j)\langle \xvec_i-  \T^{-1}  \zvec_j,  \boldsymbol{\Sigma}_{ij}^{-1} \xivec^{\wedge}(\T^{-1} \zvec_j) \rangle  \notag \\
    =&\sum_{ij} f_{ij}\rvec_{ij}^{\top}\boldsymbol{\Sigma}_{ij}^{-1} (\T^{-1}\zvec_j)^{\wedge}\omegavec - \sum_{ij}f_{ij}\rvec_{ij}^{\top}\boldsymbol{\Sigma}_{ij}^{-1}\vvec.
 \end{align}
Using~\eqref{eq:equivalence}:
\begin{align}
\D f(\T)[\xivec^\wedge]
=
\langle \grad f(\T), \xivec^\wedge \rangle
=
g_{\omegavec}^\top \omegavec
+
g_{\vvec}^\top {\vvec}.
\end{align}
and identifying the gradients for translation and rotational components from~\eqref{eq:lasteq} yields the stated result.

\section{Line search}\label{app:linesearch}
\begin{subequations}
For the Taylor expansion in the line search, we obtain expressions similar to~\protect{\cite[Sec. 7]{clarkmaani20}} except that we use anisotropic kernels,  with coefficients
\begin{align}
g^1_{ij} &= \beta_{ij}, \\
g^2_{ij} &= \gamma_{ij}+\frac{1}{2}\beta_{ij}^2, \\
g^3_{ij} &= \delta_{ij}+\beta_{ij}\gamma_{ij}+\frac{1}{6}\beta_{ij}^3, \\
g^4_{ij} &= \varepsilon_{ij} + \beta_{ij}\delta_{ij}+\frac{1}{2}\beta_{ij}^2\gamma_{ij} + \frac{1}{2}\gamma_{ij}^2 + \frac{1}{24}\beta_{ij}^4.
\end{align}
However, our expressions differ in the following parameters
\begin{align}
%\alpha_{ij} &= \frac{-1}{2}\lVert \xvec_i-\zvec_j\rVert_n \\
\beta_{ij} &= - \langle \omegavec^{\land}\T^{-1}\zvec_j+\vvec,\boldsymbol{\Sigma}_{ij}^{-1}(\xvec_i-\T^{-1} \zvec_j)\rangle,\\
%%%%%%%
\gamma_{ij} &= (-\dfrac{1}{2}) \langle \omegavec^{\land}\T^{-1}\zvec_j+\vvec, \boldsymbol{\Sigma}_{ij}^{-1} (\omegavec^{\land}\T^{-1}\zvec_j+\vvec)\rangle - \notag \\
&\hspace{10pt} \langle (\omegavec^\land)^2\T^{-1}\zvec_j+\omegavec^\land\vvec, \boldsymbol{\Sigma}_{ij}^{-1} (\xvec_i-\T^{-1} \zvec_j)\rangle, \\
%%%%%%%
\delta_{ij} &=  \Big( \langle -\omegavec^{\land}\T^{-1}\zvec_j-\vvec,  \boldsymbol{\Sigma}_{ij}^{-1}(\omegavec^{\land})^2\T^{-1}\zvec_j +\omegavec^{\land}\vvec)\rangle+\notag\\ 
&\hspace{10pt}\langle -(\omegavec^{\land})^3\T^{-1}\zvec_j-(\omegavec^{\land})^2\vvec, \boldsymbol{\Sigma}_{ij}^{-1}(\xvec_i-\T^{-1}\zvec_j)\rangle\Big),\hspace*{-5pt}\\
%%%%%%%
\varepsilon_{ij} &= -\dfrac{1}{2}\Big(\langle(\omegavec^{\land})^2\T^{-1}\zvec_j+\omegavec^{\land}\vvec, \boldsymbol{\Sigma}_{ij}^{-1}( (\omegavec^{\land})^2\T^{-1}\zvec_j+\omegavec^{\land}\vvec )\rangle+\notag\\ &
\hspace{10pt}2\langle \omegavec^{\land}\T^{-1}\zvec_j+\vvec, \boldsymbol{\Sigma}_{ij}^{-1} ((\omegavec^{\land})^3\T^{-1}\zvec_j +(\omegavec^{\land})^2\vvec)\rangle + \notag\\ &
\hspace{10pt}2\langle (\omegavec^{\land})^4\T^{-1}\zvec_j+(\omegavec^{\land})^3\vvec, \boldsymbol{\Sigma}_{ij}^{-1}  (\xvec_i-\T^{-1} \zvec_j)\rangle\Big).
\end{align}
\end{subequations}

\section{Loss Slices}\label{sec:exp:lossslices}
First, we give intuition on the behavior of the RKHS loss under different perturbations of the ground truth pose. To this end, we sample $8$ point clouds from our self-collected \lidar datasets, later described in Sec.~\ref{sec:exp:featuresparse}. Fig.~\ref{fig:ablation_loss} shows the value of the objective function when translating the target point cloud with a pose $\T=\exp(\xivec^{\land})$, where $\xivec = (0, \phi, 0, x,0, z)^{\top}$ and $x,z,\phi$ are varied independently over $[-3,3]$ for translations and $[-0.5, 0.5]$ for rotations. Near the ground truth pose of each pair of point clouds, there is a single maximum of the objective function~\eqref{eq:rkhs_loss}. Besides, the curvature of the loss is greater on the $z$ axis with the %. This is due to the 
kernel choice in~\eqref{eq:kernel}, resulting in tighter fits along surfaces. %The objective is not locally quadratic or even smooth in the rotation perturbation, motivating the line search in \ourmethod-1.

\section{Derivation of the explicit form of correction term $\Gamma(\cdot, \cdot)$}

Let
\begin{equation}
\xivec =
\begin{bmatrix}
\omegavec\\
\vvec
\end{bmatrix},
\qquad
\gvec = \grad f(\T)^\vee =
\begin{bmatrix}
g_{\omegavec}\\
g_{\vvec}
\end{bmatrix}.
\end{equation}
From~\eqref{eq:gamma_se3}, the vee representation of the connection term is
\begin{equation}
\Gamma^\vee(\xivec^\wedge,\grad f(\T))
=
\frac{1}{2}
\left(
\ad_{\xivec}\gvec
-
\ad_{\xivec}^* \gvec
-
\ad_{\gvec}^* \xivec
\right).
\label{eq:gamma_deriv_start}
\end{equation}
Using~\eqref{eq:gamma_se3},
\begin{equation}
\ad_{\xivec}
=
\begin{bmatrix}
\omegavec^\wedge & 0\\
\vvec^\wedge & \omegavec^\wedge
\end{bmatrix},
\qquad
\ad_{\xivec}^*
=
\begin{bmatrix}
-\omegavec^\wedge & -\vvec^\wedge\\
0 & -\omegavec^\wedge
\end{bmatrix}.
\end{equation}
Therefore,
\begin{equation}
\ad_{\xivec}\gvec
=
\begin{bmatrix}
\omegavec^\wedge g_{\omegavec}\\
\vvec^\wedge g_{\omegavec} + \omegavec^\wedge g_{\vvec}
\end{bmatrix},
\end{equation}
\begin{equation}
\ad_{\xivec}^* \gvec
=
\begin{bmatrix}
-\omegavec^\wedge g_{\omegavec} - \vvec^\wedge g_{\vvec}\\
-\omegavec^\wedge g_{\vvec}
\end{bmatrix},
\qquad
\ad_{\gvec}^* \xivec
=
\begin{bmatrix}
-g_{\omegavec}^\wedge \omegavec - g_{\vvec}^\wedge \vvec\\
-g_{\omegavec}^\wedge \vvec
\end{bmatrix}.
\end{equation}
Substituting into~\eqref{eq:gamma_deriv_start}, we obtain
\begin{align}
\Gamma^\vee(\xivec^\wedge,\grad f(\T))
&=
\frac{1}{2}
\begin{bmatrix}
\omegavec^\wedge g_{\omegavec}
+\omegavec^\wedge g_{\omegavec}
+\vvec^\wedge g_{\vvec}
+g_{\omegavec}^\wedge \omegavec
+g_{\vvec}^\wedge \vvec\\[4pt]
\vvec^\wedge g_{\omegavec}
+2\omegavec^\wedge g_{\vvec}
+g_{\omegavec}^\wedge \vvec
\end{bmatrix}.
\end{align}
Using \(a^\wedge b = -\,b^\wedge a\), we have
\begin{equation}
\vvec^\wedge g_{\vvec} + g_{\vvec}^\wedge \vvec = 0,
\,\,
\omegavec^\wedge g_{\omegavec} + g_{\omegavec}^\wedge \omegavec = 0,
\,\,
\vvec^\wedge g_{\omegavec} + g_{\omegavec}^\wedge \vvec = 0.
\end{equation}
Hence,
\begin{equation}
\Gamma^\vee(\xivec^\wedge,\grad f(\T))
=
\begin{bmatrix}
-\frac{1}{2} g_{\omegavec}^\wedge \omegavec\\[4pt]
- g_{\vvec}^\wedge \omegavec
\end{bmatrix}
=
\begin{bmatrix}
-\frac{1}{2} g_{\omegavec}^\wedge & 0\\
- g_{\vvec}^\wedge & 0
\end{bmatrix}
\begin{bmatrix}
\omegavec\\
\vvec
\end{bmatrix}.
\end{equation}
Therefore,
\begin{equation}
\Gamma(\xivec^\wedge,\grad f(\T))
=
\left(
\begin{bmatrix}
-\frac{1}{2} g_{\omegavec}^\wedge & 0\\
- g_{\vvec}^\wedge & 0
\end{bmatrix}
\xivec
\right)^\wedge.
\end{equation}

\section{Examples of the self-collected datasets}
Fig.~\ref{fig:oval} shows the scenes and the example point cloud observations in Sec.~\ref{sec:exp:featuresparse}. The first row is an example frame from the \texttt{Dirt Track} sequence, where there is dust floating behind the vehicle, illustrated in red in the \lidar visualization. The second row is an example frame from the \texttt{Race Track} sequence, where there are fewer features to track, and the car travels at greater speeds. The third row shows the \texttt{Skid Pad} dataset, a large paved surface.
\iffalse
\begin{figure}[t!]
    \centering
    % \includegraphics[height=1.3cm]{dirt_track.png}
    % \includegraphics[height=1.3cm]{dirt_track_illustraiton.png}
    % \includegraphics[height=1.3cm]{race_track_illustration.png}
    % \includegraphics[height=1.3cm]{th_2mile.png}
    \includegraphics[height=2.46cm]{dirt_track.png}
    \includegraphics[height=2.46cm]{dirt_track_illustraiton.png}
    \\
    \includegraphics[height=2.8cm]{race_track_illustration.png}
    \includegraphics[height=2.8cm]{th_2mile.png}
    
    \caption{Sample point clouds and images from the \texttt{Dirt Track} (left) and \texttt{Race Track} (right) environments.}
    \label{fig:samples}
\end{figure}
\fi

\begin{figure*}[h!]
\centering
\begin{tikzpicture}
\node[anchor=north west, inner sep=0] (car) at (0,0) 
%\iffalse
{
  \includegraphics[trim=0cm 0cm 0cm 0cm, clip=true, width=0.32\textwidth%height=0.2\textwidth
  ]{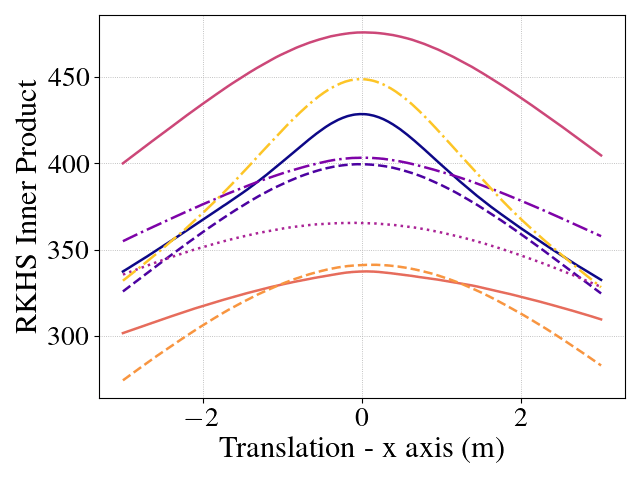}
\label{fig:ablation_convergence_no_init}
};
\node[anchor=north west, inner sep=0] (track) at (0.33\textwidth,0)
%\fi
{
  \includegraphics[trim=0cm 0cm 0cm 0cm, clip=true, width=0.32\textwidth%height=0.2\textwidth
  ]{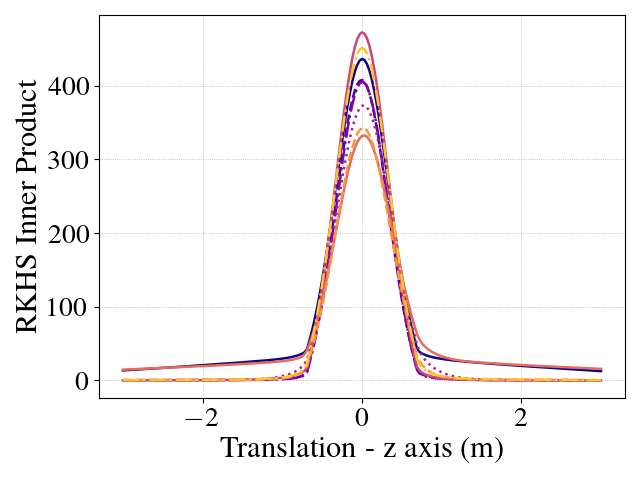}
\label{fig:ablation_convergence_init}
};
\node[anchor=north west, inner sep=0] (car) at (0.66\textwidth,0) 
{
  \includegraphics[trim=0cm 0cm 0cm 0cm, clip=true, width=0.32\textwidth]{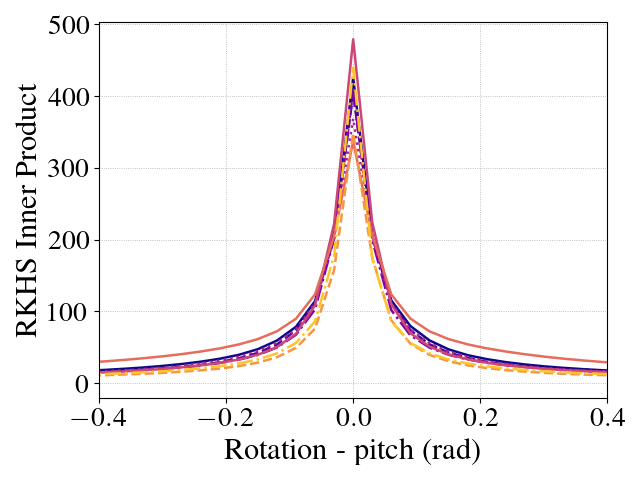}  \label{fig:ablation_convergence_no_init}
};
\end{tikzpicture}
\vspace{-10pt}

\caption{The RKHS inner product with translational and rotational perturbations. Each line corresponds to a different input pair.
}
\label{fig:ablation_loss}
\end{figure*}

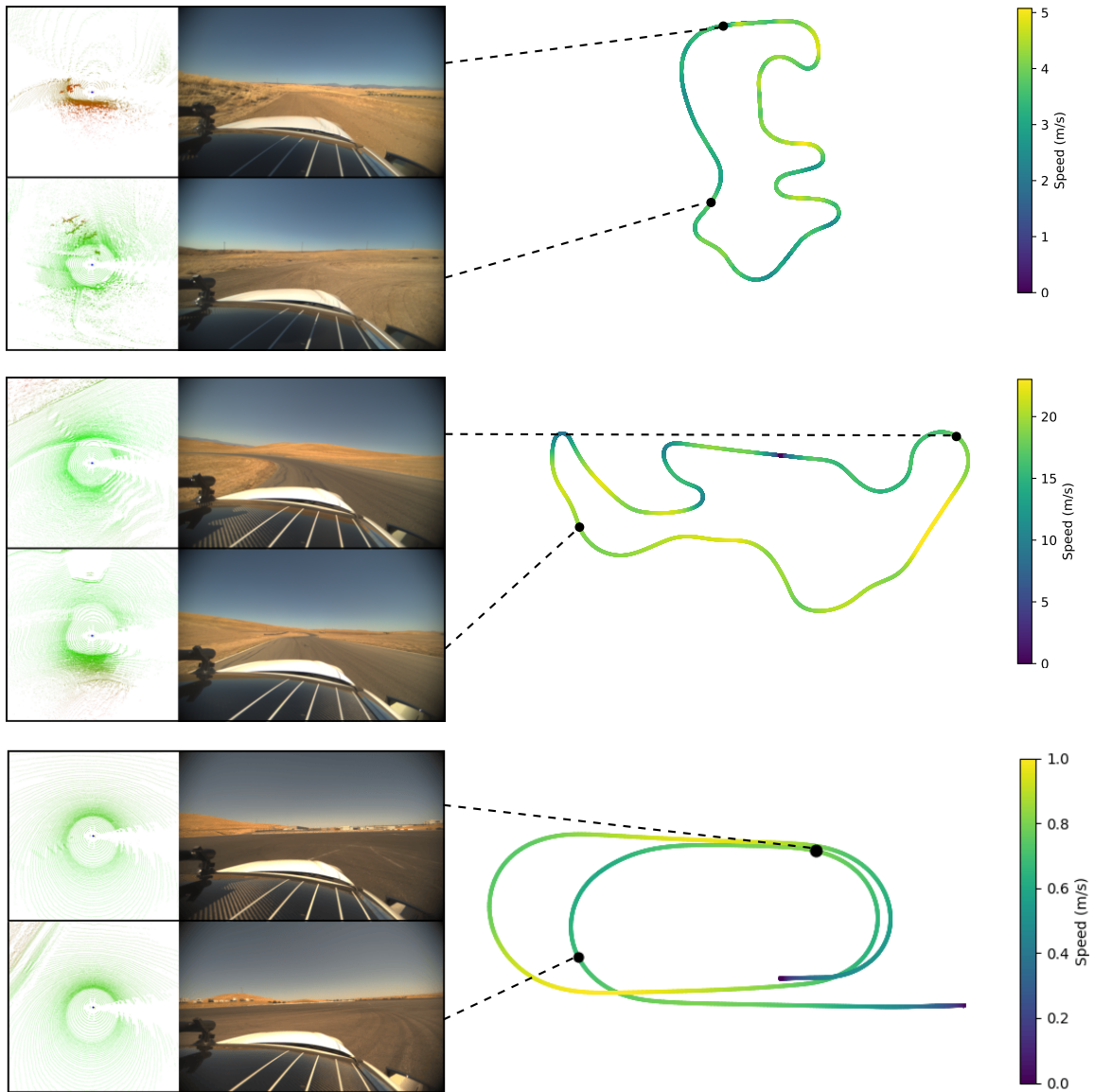
\begin{figure*}[h!]
\centering

\begin{tikzpicture}
    \node[anchor=north east] at (0,0) {\includegraphics[trim={1.3cm 1.3cm 1.3cm 1.3cm}, clip, width=7cm]{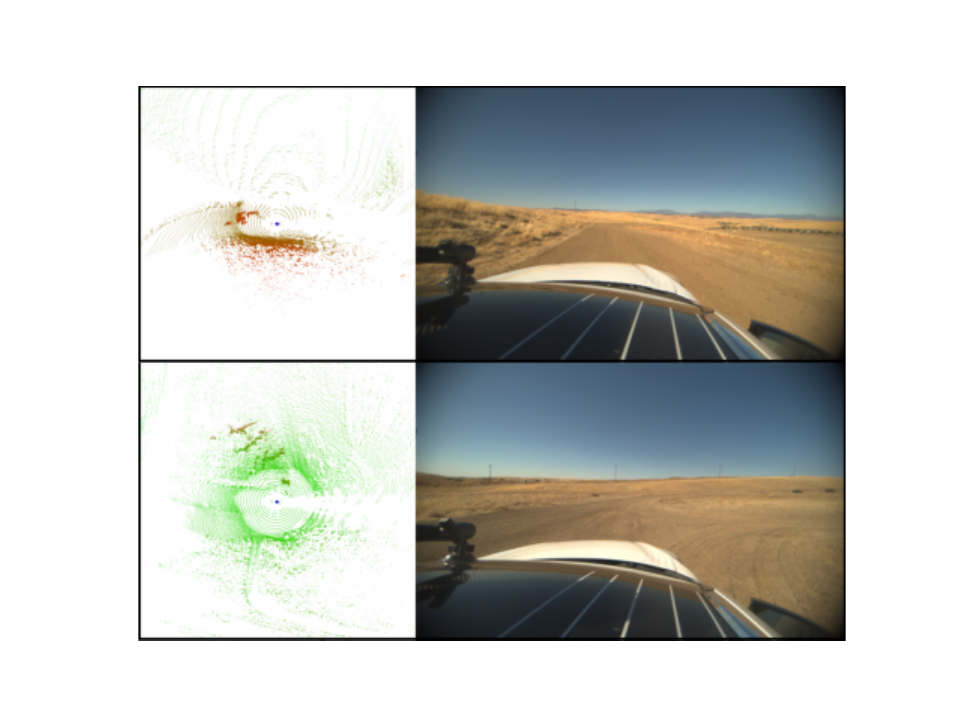}};
    \node[anchor=north west] at (0,0) {\includegraphics[trim={1.3cm 1.0cm 1.3cm 1.0cm}, clip, width=9cm]{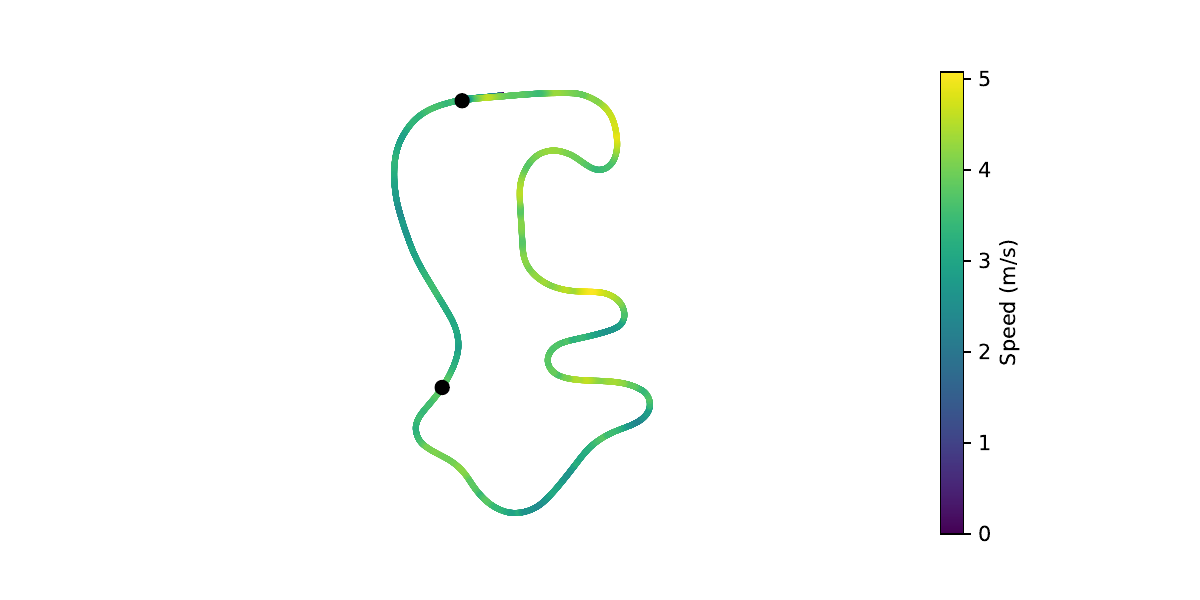}};
    \draw[black, thick, dashed] (-0.45,-1) -- (3.4,-0.5);
    \draw[black, thick, dashed] (-0.45,-4) -- (3.2,-2.95);
\end{tikzpicture}%

\begin{tikzpicture};
    \node[anchor=north east] at (0,0) {\includegraphics[trim={1.3cm 1.3cm 1.3cm 1.3cm}, clip, width=7cm]{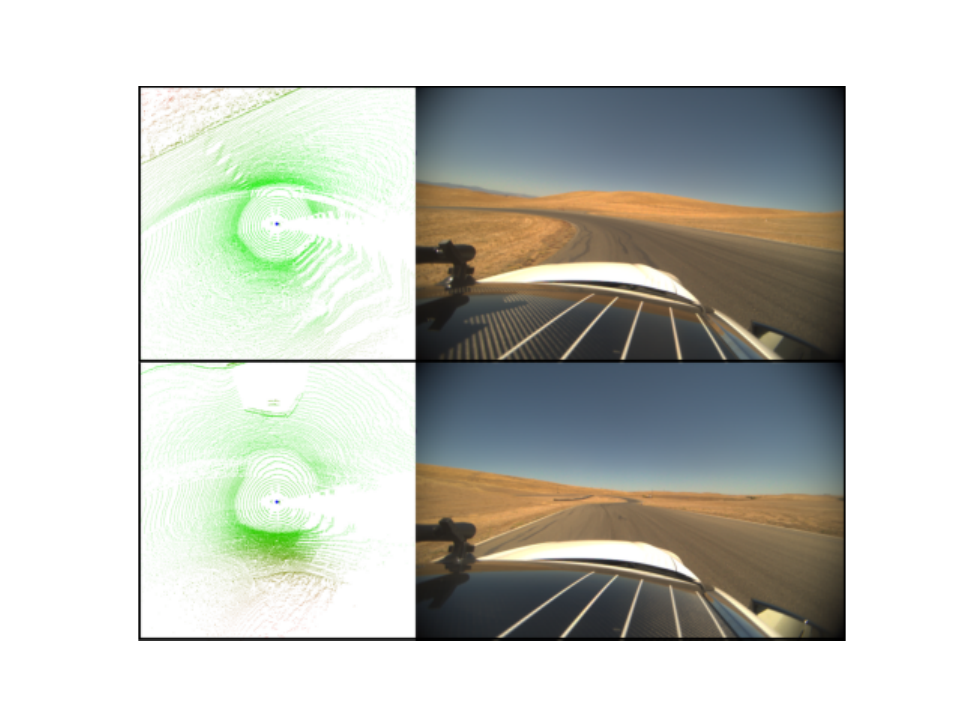}};
    \node[anchor=north west] at (0,0) {\includegraphics[trim={1.3cm 1.0cm 1.3cm 1.0cm}, clip, width=9cm]{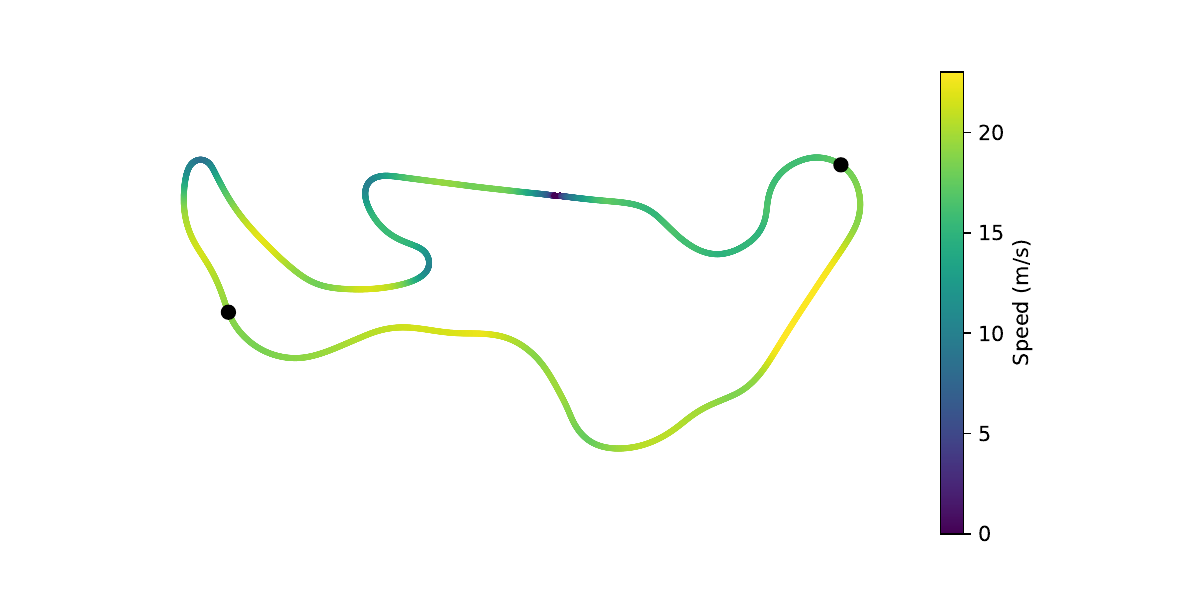}};
    \draw[black, thick, dashed] (-0.45,-1) -- (6.7,-1.02);
    \draw[black, thick, dashed] (-0.45,-4) -- (1.43,-2.3);
\end{tikzpicture}%

\begin{tikzpicture}
    \node[anchor=north east] at (0.2,0) {\includegraphics[width=6.35cm]{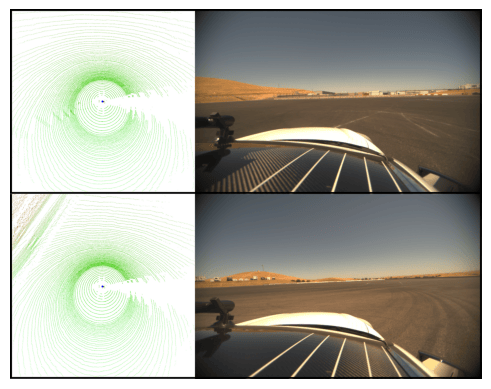}};
    \node[anchor=north west] at (0,0) {\includegraphics[width=9cm]{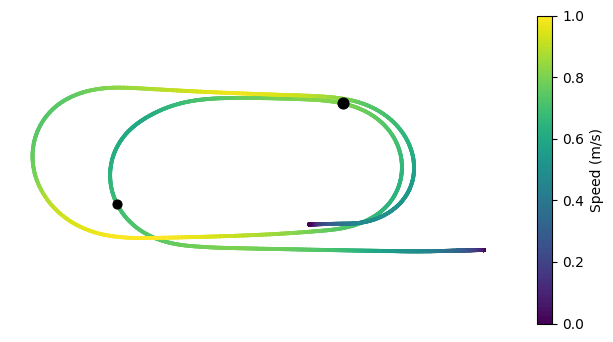}};
    \draw[black, thick, dashed] (-0.05,-1) -- (5.1,-1.6);
    \draw[black, thick, dashed] (-0.05,-4) -- (1.85,-3.1);
\end{tikzpicture}%

\vspace{-15pt}
\caption{Illustration of the \texttt{Dirt Track} dataset (top), \texttt{Race Track} dataset (middle), and \texttt{Skid Pad} dataset (bottom), with lidar point clouds and images at two points shown at two points along the track.}
\label{fig:oval}
\end{figure*}

\section{Qualitative tracking results}
We provide more qualitative results for frame-to-frame tracking on the KITTI \lidar dataset and ETH3D RGB-D dataset in Fig.~\ref{fig:kitti_stacked} and~\ref{fig:eth3d_stacked}, respectively. We stacked the point clouds based on the poses estimated from the frame-to-frame tracking results in  KITTI  \texttt{00} from Sec.~\ref{sec:exp:featurerich} and the first 400 frames of ETH3D \texttt{table\_3} from Sec.~\ref{sec:eth3d}.

% Optional: switch subfigure labels from (a),(b)... to (1),(2)...
% \renewcommand{\thesubfigure}{\arabic{subfigure}}
% \makeatletter
% \renewcommand{\p@subfigure}{\thefigure} % clean refs like Fig. 1(2)
% \makeatother

% In body
\begin{figure*}[t]
  \centering

  % Row 1 (3 across)
  \subfloat[GCVO, 2nd-order ]{%
    \includegraphics[width=0.32\textwidth]{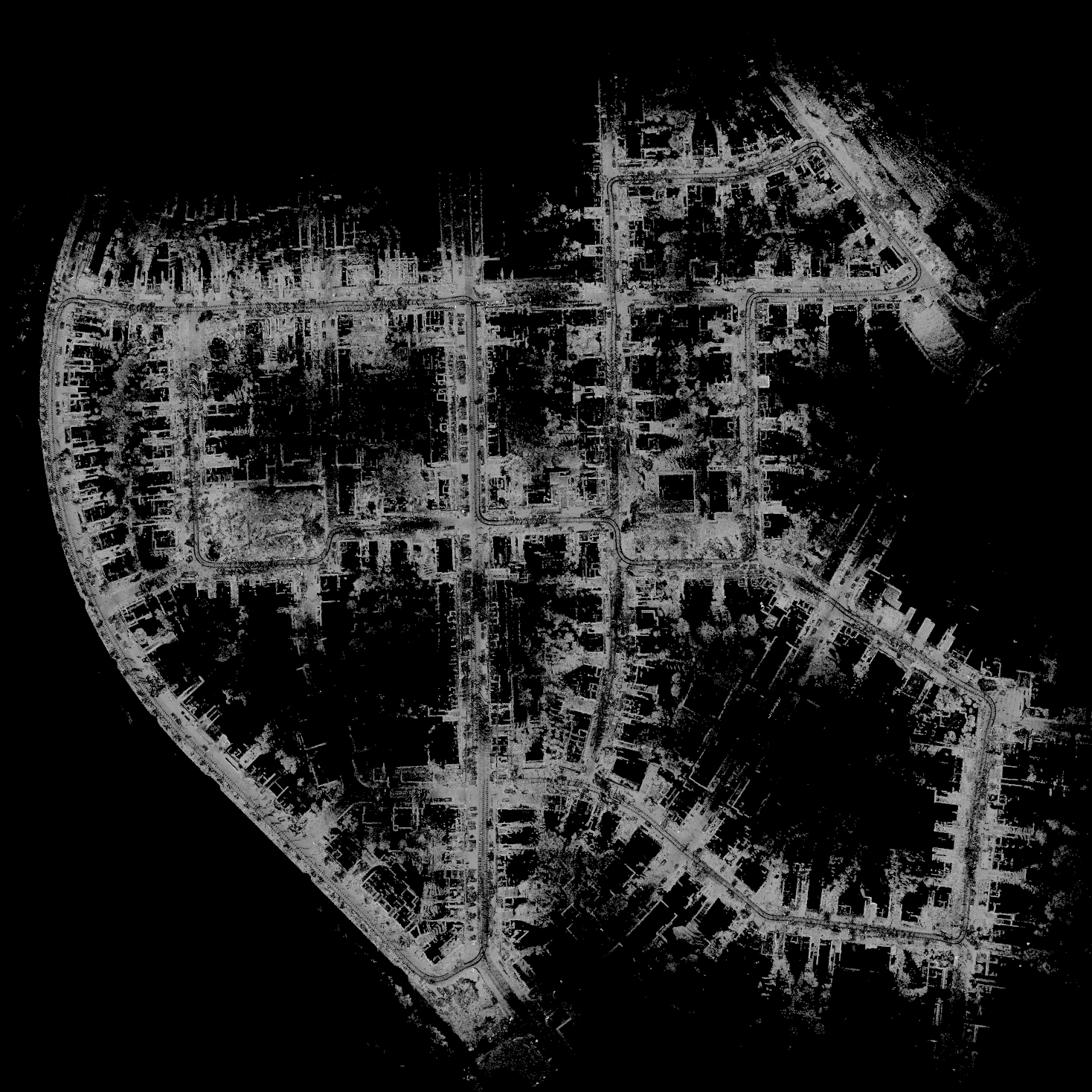}}\hfill
  \subfloat[GCVO, 1st-order ]{%
    \includegraphics[width=0.32\textwidth]{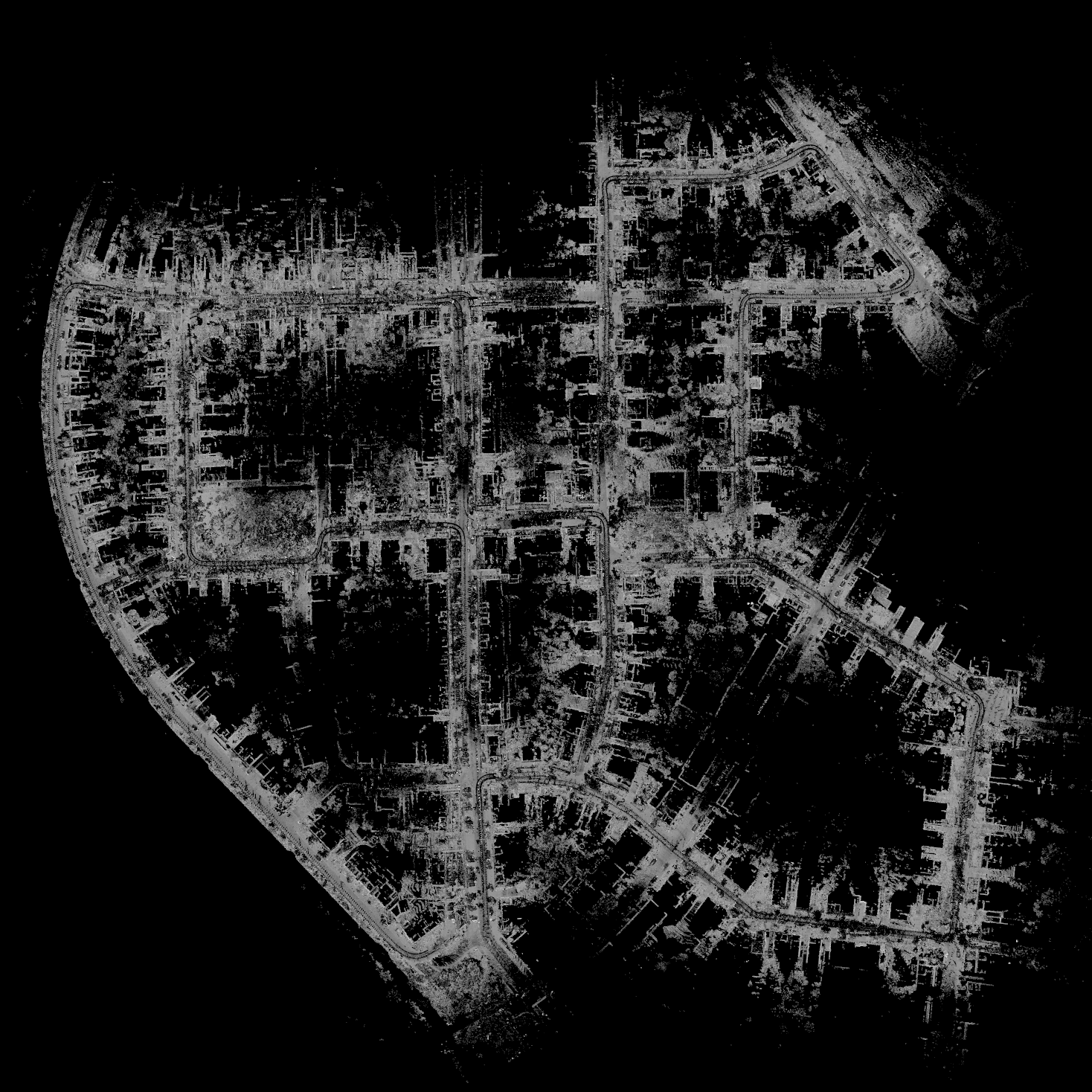}}\hfill
  \subfloat[CVO, 1st-order]{%
    \includegraphics[width=0.32\textwidth]{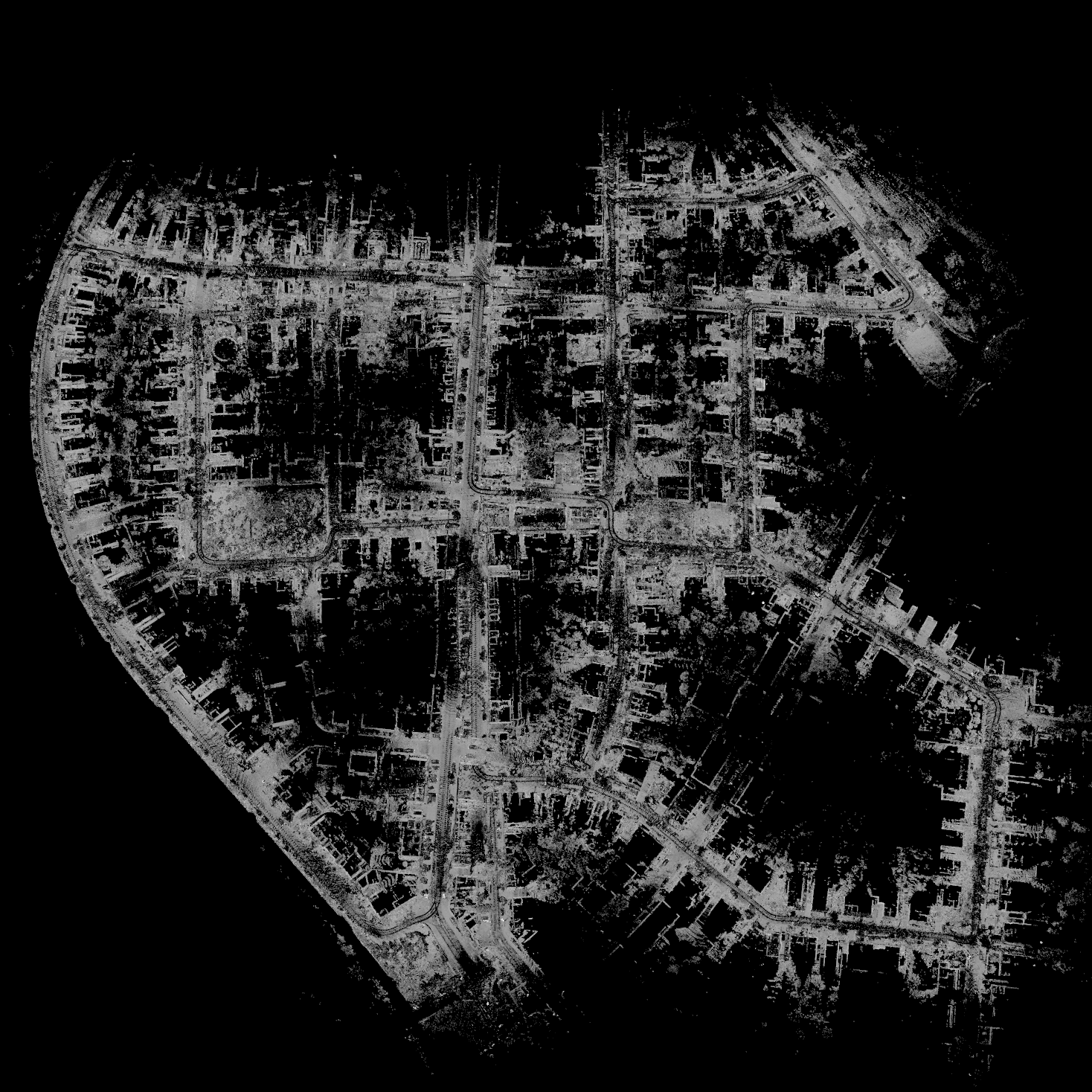}}\\[-0.2em]

  % Row 2 (3 across)
  \subfloat[Fast-VGICP]{%
    \includegraphics[width=0.32\textwidth]{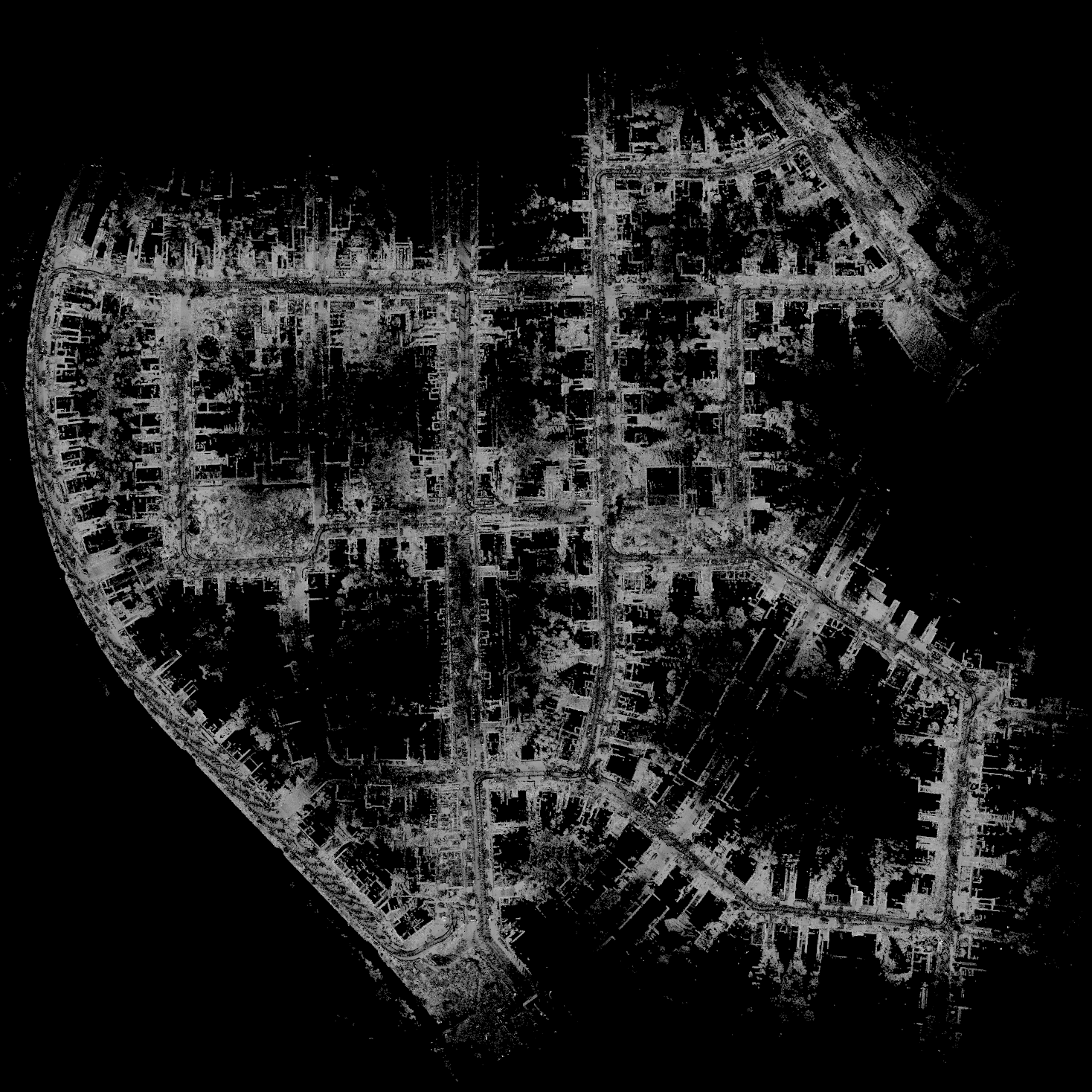}}\hfill
  \subfloat[GICP]{%
    \includegraphics[width=0.32\textwidth]{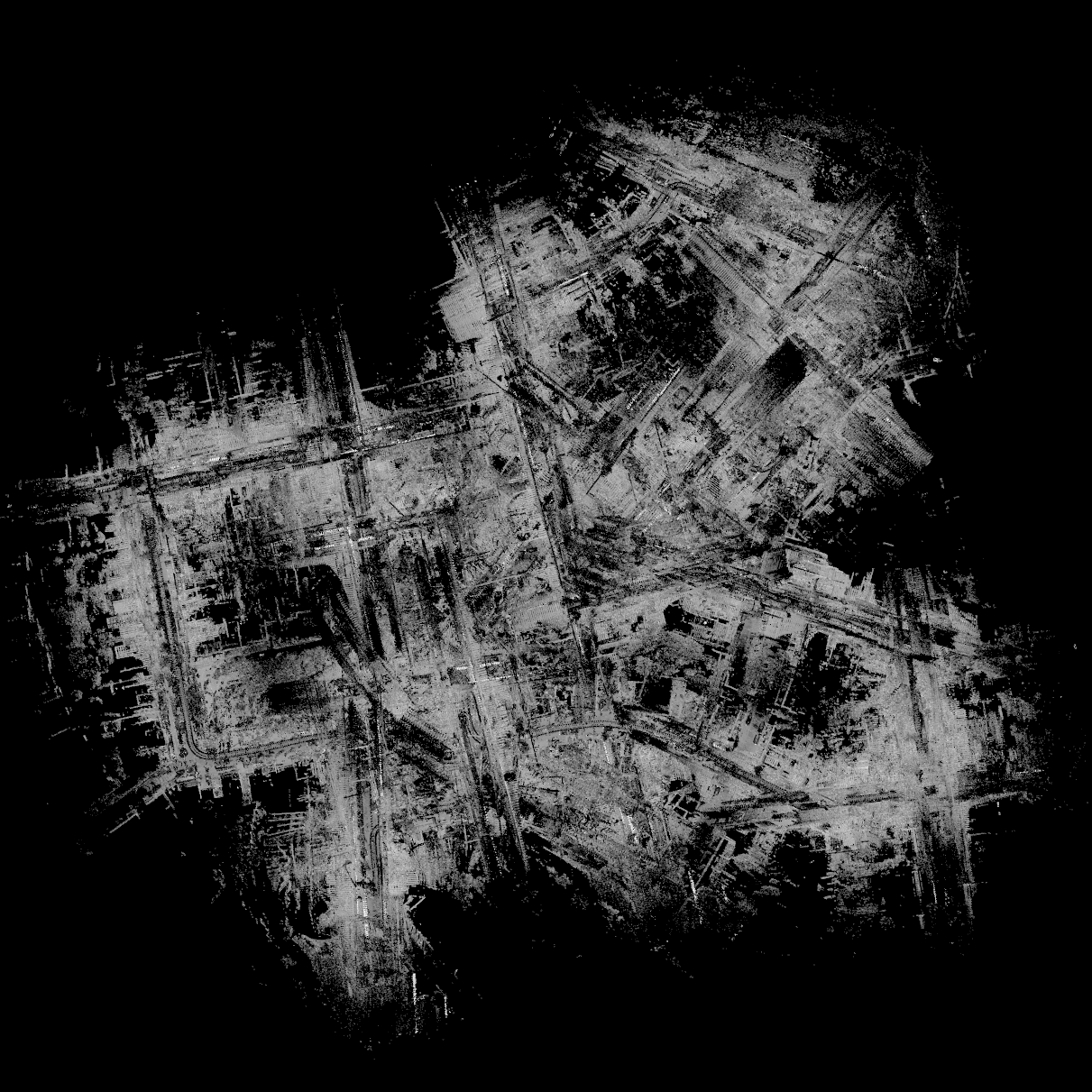}}\hfill
  \subfloat[ICP\label{fig:f}]{%
    \includegraphics[width=0.32\textwidth]{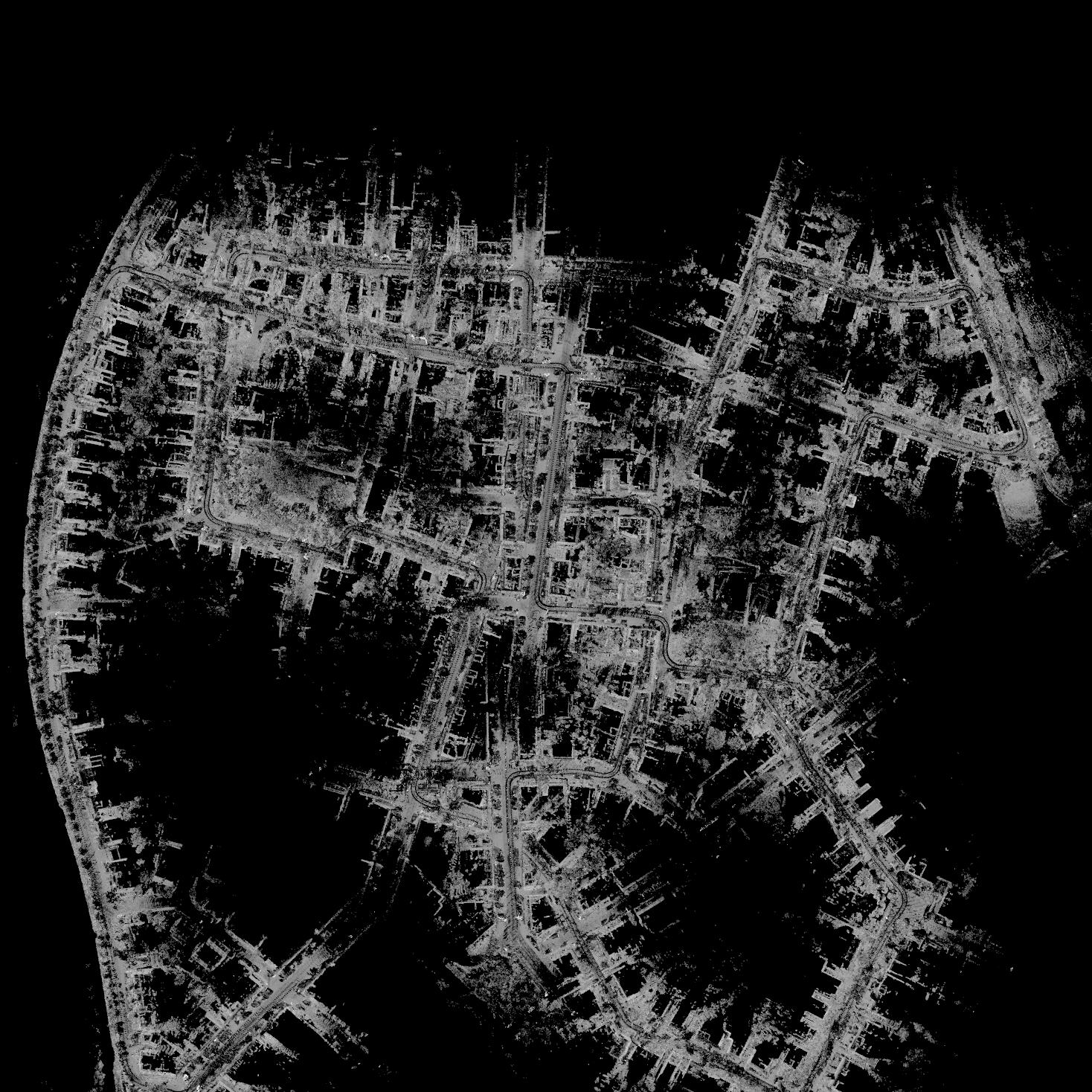}}\\[-0.2em]

  % Row 3 (single, centered)
  \hspace*{\fill}%
  \subfloat[NDT]{%
    \includegraphics[width=0.32\textwidth]{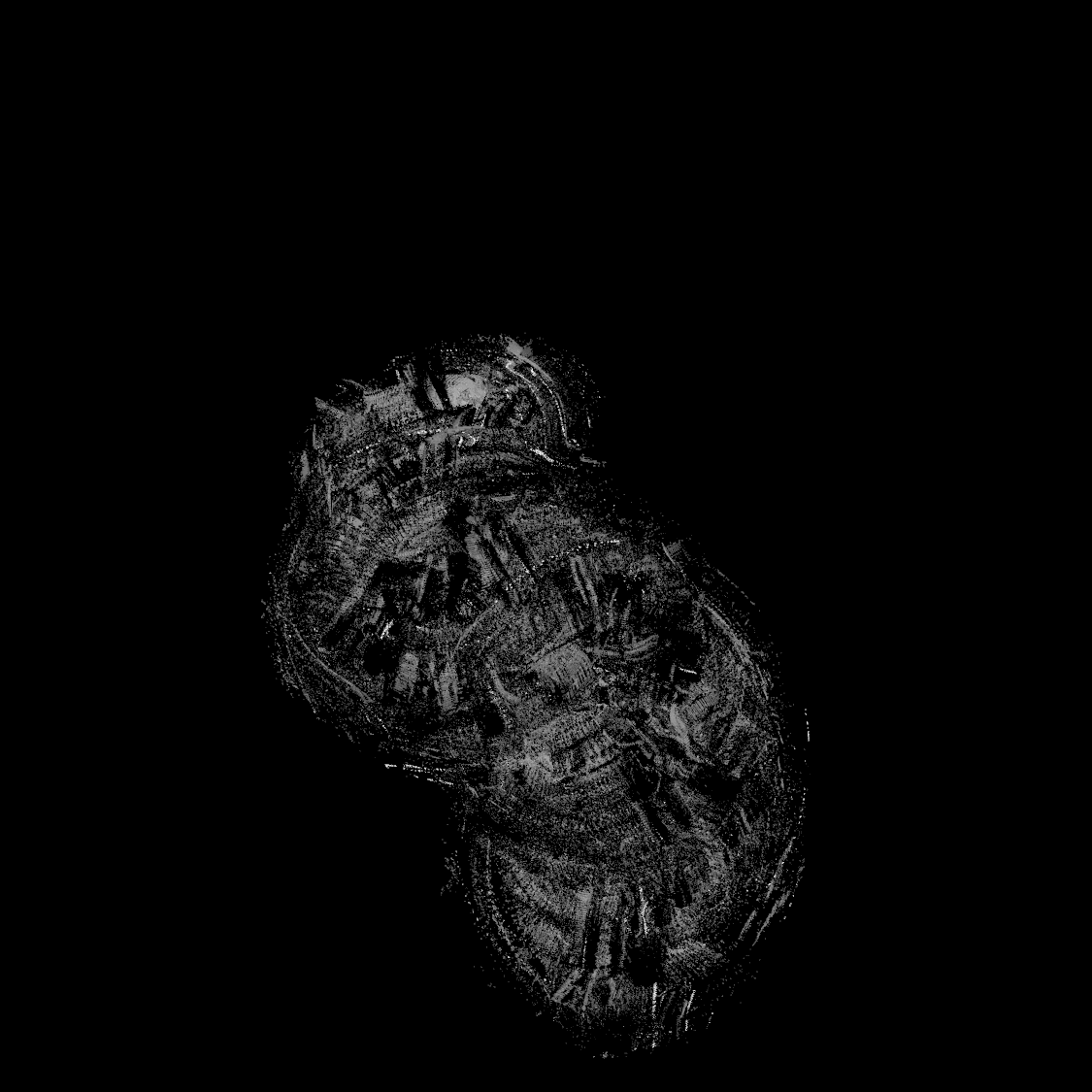}}%
  \hspace*{\fill}

  \caption{The stacked point clouds using the frame-to-frame tracking poses from KITTI \lidar sequence \texttt{00}.}
  \label{fig:kitti_stacked}
\end{figure*}

% In body

% In body
\begin{figure*}[t]
  \centering
  \newcommand{\colw}{0.31\textwidth} % tweak to 0.32 if it fits

  % -------- Row 1 (3 across, equal gaps) --------
  \makebox[\colw][c]{%
    \subfloat[GCVO, 2nd-order]{%
      \includegraphics[width=\colw]{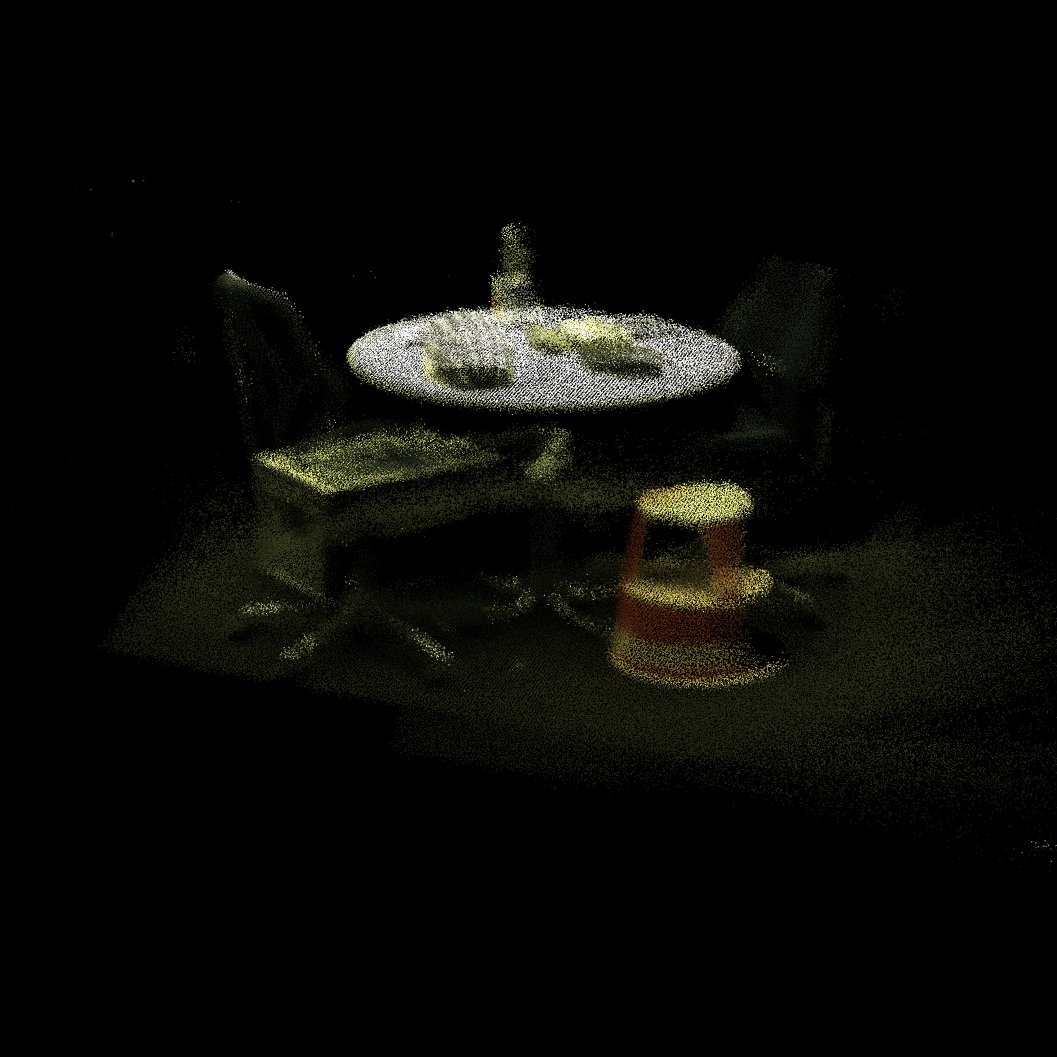}}}%
  \hfill
  \makebox[\colw][c]{%
    \subfloat[GCVO, 1st-order]{%
      \includegraphics[width=\colw]{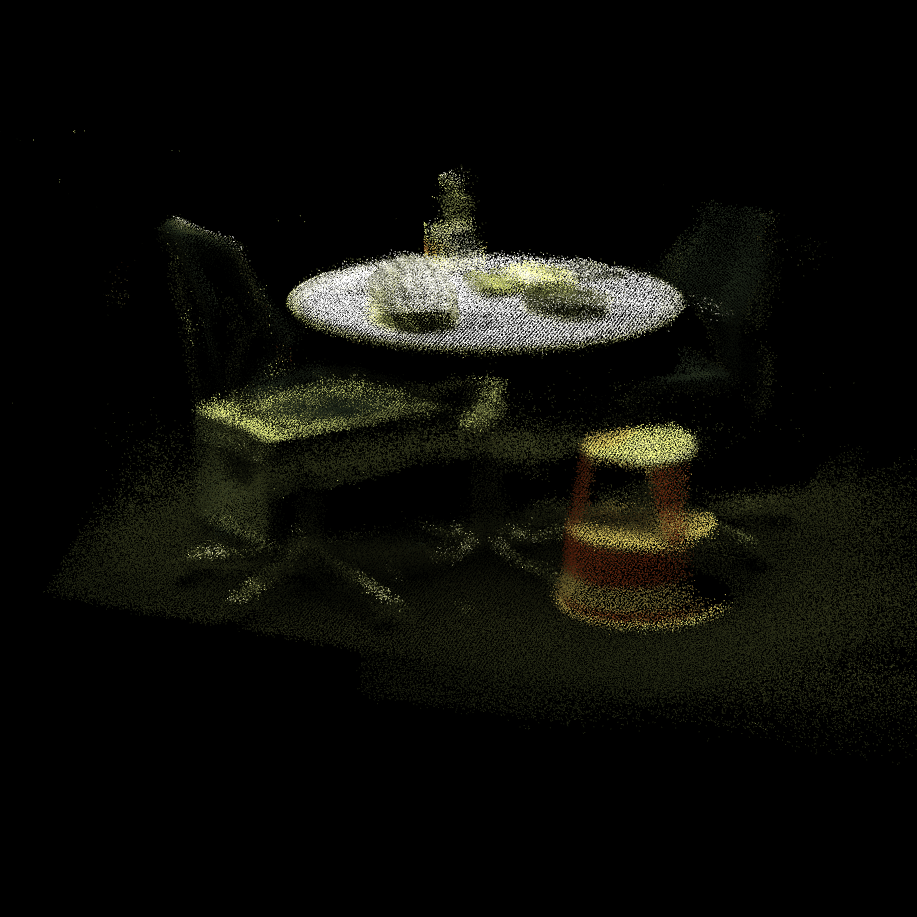}}}%
  \hfill
  \makebox[\colw][c]{%
    \subfloat[Fast-VGICP]{%
      \includegraphics[width=\colw]{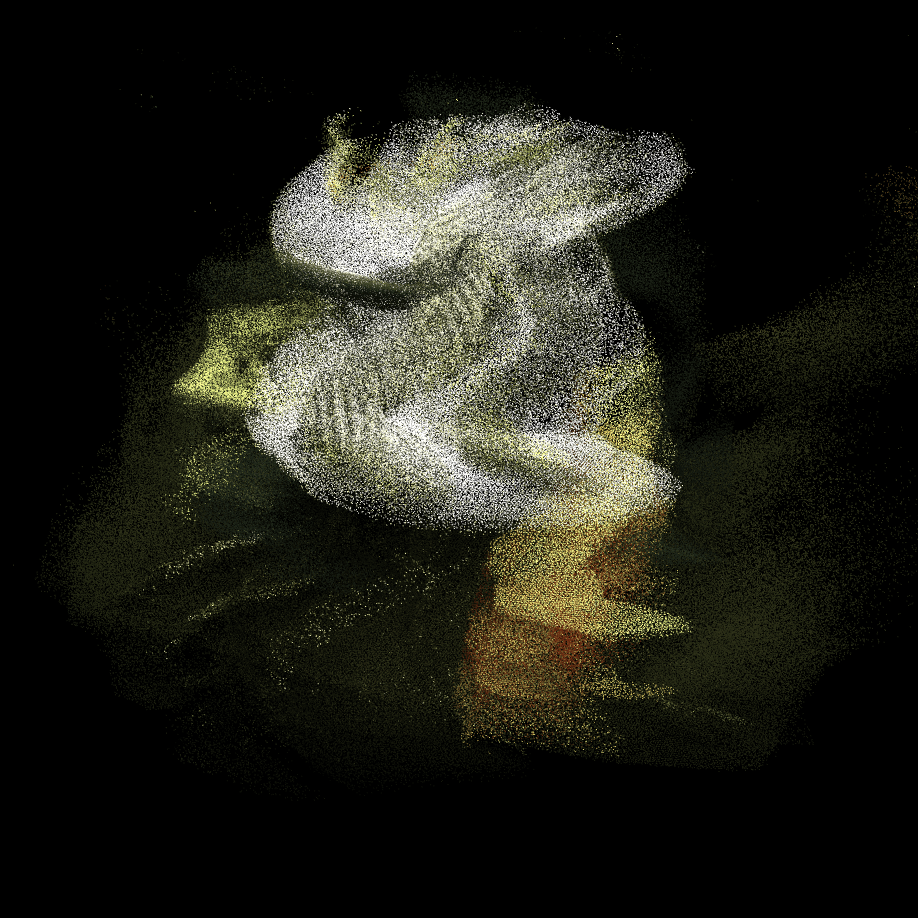}}}%

  \vspace{-0.2em}

  % -------- Row 2 (2 images + blank placeholder keeps spacing) --------
  \makebox[\colw][c]{%
    \subfloat[GICP]{%
      \includegraphics[width=\colw]{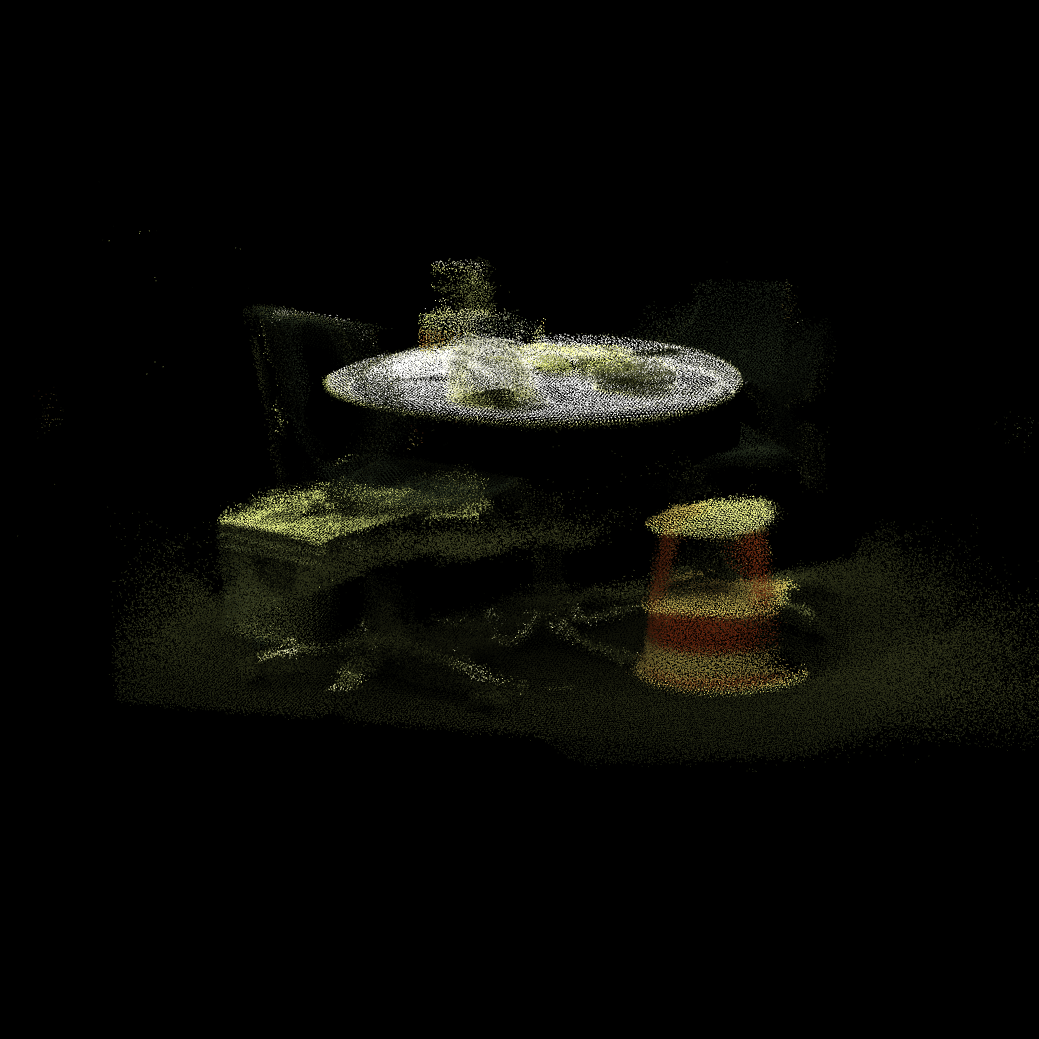}}}%
  \hfill
  \makebox[\colw][c]{%
    \subfloat[ICP]{%
      \includegraphics[width=\colw]{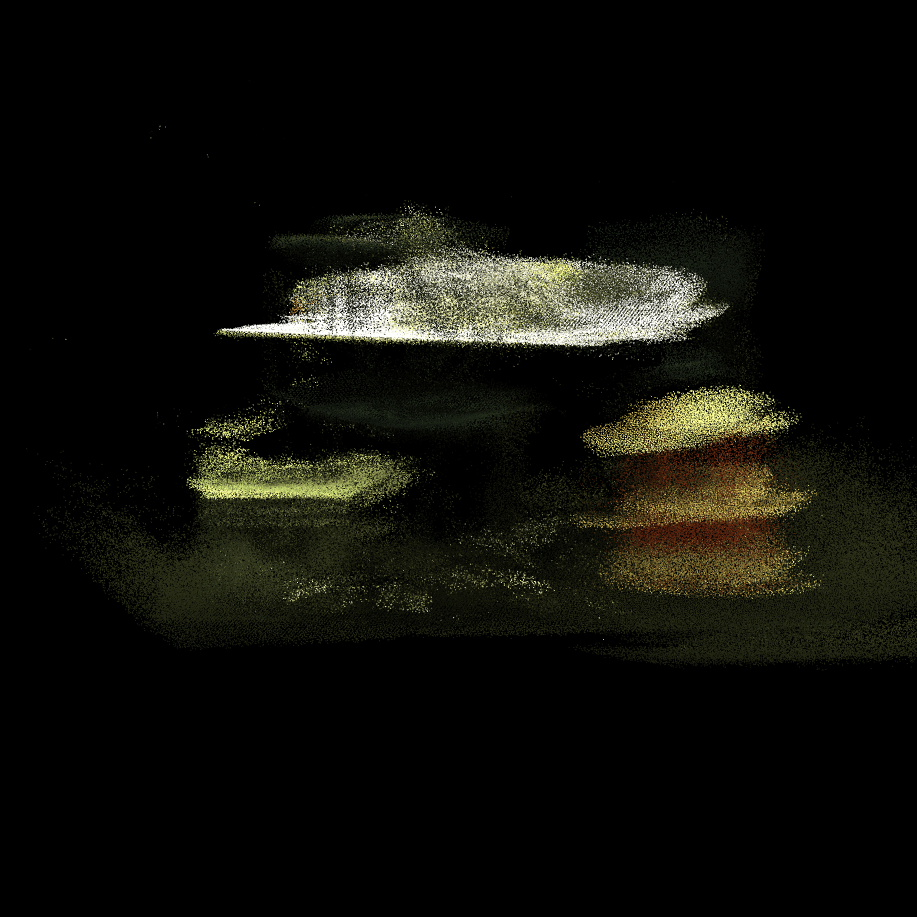}}}%
  \hfill
  % placeholder (not a subfloat, so it won't consume a label)
  \makebox[\colw][c]{\mbox{}}%

  \caption{The stacked point clouds using the frame-to-frame tracking poses from ETH3D RGB-D sequence \texttt{table\_3}.}
  \label{fig:eth3d_stacked}
\end{figure*}

\section{Ablation on point cloud density}
%\begin{wraptable}{r}{4.7cm}
\begin{table}
    \centering
    \caption{\textit{Impact of uniform vs non-uniform downsampling}}
    \vspace{-10pt}
    
    \resizebox{\columnwidth}{!}{
    \begin{NiceTabular}{c|c|c|c|c}[colortbl-like,
code-before =%
{
    % --- row 3 (sfm_bench): best/second-best per metric ---
    \cellcoloring{\bestcolor}{3}{2}
    \cellcoloring{\bestcolor}{4}{2}
    \cellcoloring{\bestcolor}{5}{2}
    %\cellcoloring{\secondbestcolor}{6}{2}
    \cellcoloring{\bestcolor}{3}{3}
    \cellcoloring{\bestcolor}{4}{3}
    %\cellcoloring{\secondbestcolor}{5}{3}
    \cellcoloring{\bestcolor}{6}{3}
    \cellcoloring{\bestcolor}{6}{4}
    \cellcoloring{\bestcolor}{5}{5}
}
]
    \toprule
    KITTI & \multicolumn{2}{c|}{G-CVO-2 w. voxel grid} &  \multicolumn{2}{c}{G-CVO-2 w. random}\\
      seq. id & trans. & rot. & trans. & rot.\\\midrule
03 &    {1.5743}	& {0.0087}	& 4.0074	& 0.0172  	       \\
05 &   {1.0557}  &	{0.0060}	   & 1.6303 & 0.0084 \\
06 &    {1.1641} &	0.0070	   & 1.7082 & {0.0065} \\ 
07 &    1.1106 &	{0.0065}	   & {1.0130} & 0.0082
\\\bottomrule
    \end{NiceTabular}}
    \label{tab:density}
\end{table}
As LiDAR inputs are generally not uniformly dense, we evaluate two types of down-sampling. We consider a voxel grid to achieve a more uniform density, and also evaluate random downsampling, which maintains non-uniformity. Tab.~\ref{tab:density} shows that G-CVO achieves higher accuracy with voxel-grid downsampling than with random downsampling, likely because the estimated covariance may be noisy with uneven point densities. Nonetheless, \ourmethod remains competitive even when point clouds are non-uniform.

%\fi
\end{document}